\documentclass{article}
\usepackage[utf8]{inputenc}
\usepackage{proceed}
\usepackage{algorithm}
\usepackage{algorithmic}

\usepackage{epsf}
\usepackage{amsmath}
\usepackage{amsthm}
\usepackage{graphicx}
\usepackage{color}
\usepackage{amssymb}
\usepackage{wrapfig}
\usepackage{natbib}

\newcounter{lem_counter}
\newcounter{pro_counter}
\newtheorem{proposition}[pro_counter]{Proposition}%[section]
\newtheorem{lemma}[lem_counter]{Lemma}%[Lemma]
\usepackage[finalold]{trackchanges}
\addeditor{lb}
\addeditor{wd}
\addeditor{dm}
\title{ O$^2$TD: (Near)-Optimal Off-Policy TD Learning}
\author{
{Bo Liu} \\
{Auburn University}\\
{boliu@auburn.edu}
\And
{Daoming Lyu} \\
{Auburn University}\\
{dzl0053@auburn.edu}
\And
{Wen Dong}\\
{University of Buffalo}\\
{wendong@buffalo.edu}
\And
{Saad Biaz}\\
{Auburn University}\\
{biazsaa@auburn.edu}
}

\begin{document}

\maketitle
\begin{abstract}
Temporal difference learning and Residual Gradient methods are the most widely used temporal difference based learning algorithms; however, it has been shown that none of their objective functions is optimal w.r.t approximating the true value function $V$. Two novel algorithms are proposed to approximate the true value function $V$. This paper makes the following contributions:
\begin{itemize}
\item A batch algorithm that can help find the approximate optimal off-policy prediction of the true value function $V$.
\item A linear computational cost (per step) near-optimal algorithm that can learn from a collection of off-policy samples.
\item A new perspective of the emphatic temporal difference learning which bridges the gap between off-policy optimality and off-policy stability.

\end{itemize}
\end{abstract}
%% old
% This work also shows that the approximate optimal frameworks subsume several previous algorithms, including the emphatic TD learning algorithm, etc. 

% that are not necessarily ergodic/sequential.

% suggested that the emphatic TD algorithm is not only off-policy stable but is also approximately off-policy optimal. The gap between off-policy optimality and off-policy stability is bridged.

%\item Shed a new light on the emphatic TD learning, i.e., the results shows that emphatic TD is not only off-policy convergent but can also improve the vanilla TD algorithm's solution property.
%\item Propose novel batch learning algorithms to obtain the approximately optimal solution. The batch learning algorithms also subsume existing algorithms and inspire many novel learning algorithms.
%\item Build up an off-policy solution analysis framework using importance sampling ratio and oblique projection.

\section{Introduction}

Temporal difference (\textbf{TD}) learning is a widely used method in reinforcement learning. There are two fundamental problems in temporal difference learning. 
The \textbf{first problem} is the off-policy stability. Although TD converges when samples are drawn “on-policy” by sampling from the Markov chain underlying a policy in a Markov decision process, it can be shown to be divergent when samples are drawn “off-policy.” Off-policy stable methods are of wider applications since they can learn while executing an exploratory policy, learn from demonstrations, and learn multiple tasks in parallel. 
The \textbf{second problem} is the optimality with function approximation. An accurate prediction of the value function will greatly help improve the policy optimization, which is the ultimate goal of reinforcement learning tasks. On the other hand, a bad value function prediction will lead to a low-quality policy \citep{sutton-barto:book}.

%%%% off-policy
Several different approaches have been explored to address the problem of
off-policy temporal difference learning. Baird's residual gradient (\textbf{RG}) method \citep{Baird:ResidualAlgorithms1995} is the first approach with linear complexity per step, but it requires double sampling and also converges to an inferior solution. 
 \cite{gordon1996stable} proposed the ``averager'' method,  which needs to store many training examples, and thus is not practical for large-scale applications. The off-policy LSTD \citep{yu2010:icml} is off-policy convergent, but its per-step computational complexity
is quadratic in the number of parameters $d$ of the function approximator.
\cite{Sutton:GTD1:2008,tdc:2009} proposed the family of gradient-based temporal difference (\textbf{GTD}) algorithms which are proven to be asymptotically off-policy convergent using stochastic approximation~\citep{borkar:book}.
% The algorithms were later extended to proximal gradient TD  \citep{liu2015uai} (PGTD) by formulating the algorithm as a stochastic primal-dual methods, and sample complexity analysis, regularization, acceleration are also given based on off-the-shelf saddle-point solvers.
% Recently, an off-policy learning Q-learning based on a clipped importance ratio method, called Retrace($\lambda$), was proposed by \cite{saferl:munos2016safe}.

%%%% optimal
Another direction of temporal difference learning, optimal temporal difference learning, seems to draw relatively insufficient attention. 
It is well-known that the asymptotic solutions of TD and GTD are not the true value function $V$, but the solution of a projected fixed point equation \citep{tdc:2009}. On the other hand, the residual gradient method converges to another solution, which is often inferior to the TD solution. However, as pointed out by \cite{Scherrer:ObliqueProjection},  both the TD and residual gradient method can be unified as the oblique projection of the true value function $V$ with different oblique projection directions, and \textit{neither }of them is optimal in the sense of approximating the true value function $V$.
To the best of our knowledge, the most relevant to our work is the optimal Dantzig Selector TD learning \citep{liu2016dantzig}, which aims to find the best denoising matrix for the purpose of feature selection, when the number of samples $n$ is much larger than the number of function approximation parameters $d$.

%%%% contribution
This paper attempts to improve the prediction of value function based on the technique of oblique projection. Here is a roadmap for the rest of the paper.
Section~\ref{sec:problem} introduces the relationship between the optimal approximation of the true value function $V$ with the oblique projected fixed point equations, which reduces the problem to finding the optimal oblique projection direction. Unfortunately, this cannot be directly computed. To this end, Section~\ref{sec:alg} proposes an approximation criterion and two algorithms, i.e., a state-aggregated batch algorithm and a state-weighted stochastic algorithm. Related work is discussed in Section~\ref{sec:related}.
Section~\ref{sec:experimental} presents the experimental results evaluating the effectiveness of the proposed approaches.

%%%%%%%%%%%%%%%%%%%%%%%%%%%%%%%%%%%%%%%%
%%%%%%%%%%%%%%%%%%%%%%%%%%%%%%%%%%%%%%%%
%%%%%%%%%%%%%%%%%%%%%%%%%%%%%%%%%%%%%%%%

\section{Preliminary}

Reinforcement Learning (\textbf{RL})~\citep{ndp:book,sutton-barto:book} is a class of learning problems in which an agent interacts with an unfamiliar, dynamic, and stochastic environment, where the agent's goal is to optimize some measure of its long-term performance. This interaction is conventionally modeled as a Markov decision process (\textbf{MDP}). An MDP is defined as the tuple $({\mathcal{S},\mathcal{A},P_{ss'}^{a},R,\gamma})$, where $\mathcal{S}$ and $\mathcal{A}$ are the sets of states and actions, the transition kernel $P_{ss'}^{a}$ specifying the probability of transition from state $s\in\mathcal{S}$ to state $s'\in\mathcal{S}$ by taking action $a\in\mathcal{A}$, $R(s,a):\mathcal{S}\times\mathcal{A}\to\mathbb{R}$ is the reward function bounded by $R_{\max}$., and $0\leq\gamma<1$ is a discount factor. A stationary policy $\pi:\mathcal{S}\times\mathcal{A}\to\left[{0,1}\right]$ is a probabilistic mapping from states to actions. The main objective of a RL algorithm is to find an optimal policy. In order to achieve this goal, a key step in many algorithms is to calculate the value function of a given policy $\pi$, i.e.,~$V^{\pi}:\mathcal{S}\to\mathbb{R}$, a process known as {\em policy evaluation}. It is known that $V^\pi$ is the unique fixed-point of the {\em Bellman operator} $T^\pi$, i.e.,
\begin{equation}
\label{eq:BellmanEq}
V^\pi = T^\pi V^\pi = R^\pi + \gamma P^\pi V^\pi,
\end{equation}
where $R^\pi$ and $P^\pi$ are respectively the reward function and transition kernel of the Markov chain induced by policy $\pi$. In Eq.~\ref{eq:BellmanEq}, we may think of $V^\pi$ as a $|\mathcal{S}|$-dimensional vector and write everything in vector/matrix form. We also denote $L^\pi := I - \gamma {P^\pi }$.
In the following, to simplify the notation, we often drop the dependence of $T^\pi$, $V^\pi$, $R^\pi$, and $P^\pi$ to $\pi$. 

We denote by $\pi_b$, the behavior policy that generates the data, and by $\pi$, the target policy that we would like to evaluate. They are the same in the on-policy setting and different in the off-policy scenario. For each state-action pair $(s_i,a_i)$, such that $\pi_b(a_i|s_i)>0$, we define the importance-weighting factor $\rho_i = \pi(a_i|s_i)/\pi _b(a_i|s_i)$ with $\rho_{\max}\geq 0$ being its maximum value over the state-action pairs.

When $\mathcal{S}$ is large or infinite, we often use a linear approximation architecture for $V^\pi$ with parameters $\theta\in\mathbb{R}^d$ and $K$-bounded basis functions $\{\varphi_i\}_{i=1}^d$, i.e.,~$\varphi_i:\mathcal{S}\rightarrow\mathbb{R}$ and $\max_i||\varphi_i||_\infty\leq K$. We denote by $\phi(\cdot) := \big(\varphi_1(\cdot),\ldots,\varphi_d(\cdot)\big)^\top$ the feature vector and by $\mathcal{F}$ the linear function space spanned by the basis functions $\{\varphi_i\}_{i=1}^d$, i.e.,~$\mathcal{F}=\big\{f_\theta\mid\theta\in\mathbb{R}^d\;\text{and}\;f_\theta(\cdot)=\phi(\cdot)^\top\theta\big\}$. We may write the approximation of $V$ in $\mathcal{F}$ in the vector form as $\hat{v}=\Phi\theta$, where $\Phi$ is the $|\mathcal{S}|\times d$ feature matrix, and we denote 
\begin{equation}
\Delta := (I - \gamma {P^\pi })\Phi  = {L^\pi }\Phi.
\label{eq:ac}
\end{equation}
When only $n$ training samples of the form $\mathcal{D}=\big\{\big(s_i,a_i,r_i=r(s_i,a_i),s'_i\big)\big\}_{i=1}^n,\;s_i\sim\xi,\;a_i\sim\pi_b(\cdot|s_i),\;s'_i\sim P(\cdot|s_i,a_i)$, are available ($\xi$ is a vector representing the probability distribution over the state space $\mathcal{S}$), 
%%%%%%%%%%%%%%%
%% comment this part, since we use state aggregation here.
% we may write the {\em empirical Bellman operator} $\hat{T}$ for a function in $\mathcal{F}$ as 
% %
% \begin{equation}
% \label{eq:EmpBellmanEq}
% \hat{T}(\hat \Phi \theta ) = \hat R + \gamma\hat\Phi '\theta,
% \end{equation}
% %
% % \note[wd]{What are $\hat\Phi$ and $\hat\Phi'$? Should they be $\Phi$ and $\hat\Phi'$?} 
% where $\hat{\Phi}$ (resp.~$\hat{\Phi}'$) is the empirical feature matrix of size $n\times d$, whose $i$-th row is the feature vector $\phi(s_i)^\top$ (resp.~$\phi(s'_i)^\top$), and $\hat{R}\in\mathbb{R}^n$ is the reward vector, whose $i$-th element is $r_i$. 
%%%%%%%%%%%%%%%%%%
we denote by $\delta_i(\theta) := r_i+\gamma\phi_i^{'\top}\theta-\phi_i^\top\theta$, the TD error for the $i$-th sample $(s_i,r_i,s'_i)$ and define $\Delta\phi_i=\rho_i (\phi_i-\gamma\phi'_i)$. 
We also denote the sample-based state-aggregated estimation of $\Delta$ (resp. $R$), termed as $\hat{\Delta}$ (resp. $\hat{R}$), i.e.,  given sample set $\mathcal{D}$,  the $i$-th and $j$-th samples are aggregated if $s_i=s_j$, which is a standard approach used in state aggregation methods \citep{stateagg:singh1995reinforcement}.
Finally, we define the matrices $C$ as $C := \mathbb{E}[\phi_i\phi_i^\top]$, where the expectations are w.r.t.~$\xi$ and $P^{\pi_b}$. We also denote by $\Xi$, the diagonal matrix whose elements are $\xi(s)$, and ${\xi _{\max }} := {\max _s}\xi (s)$. For each sample $i$ in the training set $\mathcal{D}$, the unbiased estimate of $C$ is $\hat{C}_i := \phi_i\phi_i^\top$.

%%%%%%%%%%%%%%%%%%%%%%%%%%%%%%%%%%%%%%%%
%%%%%%%%%%%%%%%%%%%%%%%%%%%%%%%%%%%%%%%%
%%%%%%%%%%%%%%%%%%%%%%%%%%%%%%%%%%%%%%%%

%In the seminal paper by \cite{Scherrer:ObliqueProjection}, the author did not propose the algorithm to find the optimal solution or give analysis on the off-policy analysis.

%%%%%%%%%%%%%%%%%%%%%%%%%%%%%%%%%%%%%%%%
%%%%%%%%%%%%%%%%%%%%%%%%%%%%%%%%%%%%%%%%
%%%%%%%%%%%%%%%%%%%%%%%%%%%%%%%%%%%%%%%%
\section{Problem Formulation}
\label{sec:problem}
This section presents the motivation of this research, i.e., exploring the possible optimal value function approximation in a model-free reinforcement learning setting. 

It is evident that given the functional space $\mathcal{F}$ and the approximation of $V$ in $\mathcal{F}$ in the vector form represented as $\hat{v}=\Phi\theta$, the optimal approximation is $ v^* = \Pi V$, where $\Pi  = \Phi {({\Phi ^ \top }\Xi \Phi )^{ - 1}}{\Phi ^ \top }\Xi $ is the weighted least-squares projection weighted by the state distribution $\Xi$. This is obtained  from $\arg \mathop {\min }\limits_{\hat v}  ||\hat v - V||_\xi^2$. 
% Taking partial derivative $\frac{\partial}{\partial{\hat v}}=\Phi^\top\Xi(\hat v-V)\overset{\text{set}}=0$. 
% It follows that $\Phi^\top\Xi\Phi\theta=\Phi^\top\Xi V$ and $\theta=\left(\Phi^\top\Xi\Phi\right)^{-1}\Phi^\top\Xi V$.
It is also well-known that the TD solution $\hat{v}_{TD}$ does not converge to $v^*$ but to the unique fixed-point solution of $\hat v = \Pi T\hat v$.  
Several intuitive questions arise here, such as 
\begin{enumerate}
\item What is the approximation error bound between $\hat{v}_{TD}$ and $V$?
\item What is the relation of representation between $\hat{v}_{TD}$ and $V$, i.e., if $\hat{v}_{TD}$ can be analytically represented by $V$? 
\end{enumerate}

The first question has been answered in \citep{tsitsiklis-roy:tdfun}, where an upper bound was given as 
$
||V - {{\hat v}_{TD}}|{|_\xi } \le \frac{1}{{\sqrt {1 - {\gamma ^2}} }}||V - {v^*}|{|_\xi }.
$ 
The answer to the second question requires the notion of oblique projection defined in Section~\ref{sec:ob}. 

\subsection{Oblique Projection and Optimal Projection}
\label{sec:ob}

%%% We first introduce the definition of weighted projected Bellman error (WPBE) as defined in \cite{emphatic:mahmood2015},
%%%\begin{align}
%%%{\rm{WPBE}} = \Phi {({\Phi ^\top}X\Phi )^{ - 1}}{\Phi ^\top}X(T\hat v - \hat v),
%%%\end{align}
%%%where $X$ is the weight matrix. It is obvious that for TD, GTD, and proximal TD algorithm, we hav e$X = \Xi$. Actually, $X$ is not necessarily $\Xi$. Then an intuitive question is: What is the criteria to choose $X$? And what is the optimal $X^*$ under such criteria?
%%%%%%%%%%%%%%%%%%%%%

%%%%%% OP
 The oblique projection tuple ($\Phi, X$) is
defined as follows, where $X$ is a matrix with the same size as $\Phi$.

\textbf{Definition.}  
The \textit{Oblique Projection} operator $\Pi _\Phi ^X$ 
is defined as 
\begin{equation}
\Pi _\Phi ^X = \Phi {({X^ \top }\Phi )^{ - 1}}{X^ \top },
\end{equation}
which specifies a projection orthogonal to $span(X)$ and onto $span(\Phi)$. It can be easily deducted that the weighted orthogonal projection $\Phi$ can be written as $\Pi  = \Pi _\Phi ^{\Xi \Phi}$.

It is easy to verify that the projected fixed point equation in temporal difference learning, $\hat v = \Pi {T^\pi }(\hat v)$, can be extended to a more general setting by extending the weighted least-squares projection operator to oblique projection operator as
\begin{align}
\hat v = \Pi _\Phi ^X{T^\pi }(\hat v),
\label{eq:obfp}
\end{align}
It turns out that both TD and RG solutions are oblique projections with different $X$, where ${X_{TD}} = \Xi \Phi ,{X_{RG}} = \Xi L^\pi \Phi $ \citep{Scherrer:ObliqueProjection}.
One may be interested in the relation between the true value function $V$ and the solutions of the fixed-point equation. The relation is shown in Lemma~\ref{lem:1}.
% \note[wd]{I get stuck here: $\hat v$ is the fixed point in the lemma and I cannot find a trick to introduce $V$ into the identity. Are we trying to solve this $X$? $\Pi_\Phi^{L^{\pi\top}X}V=\Phi\left(X^\top L^\pi\Phi\right)^{-1}X^\top\underbrace{(I-\gamma P^\pi)V}_R=?$}
\begin{lemma}\citep{Scherrer:ObliqueProjection}
The solution of the oblique projected fixed-point equation $\hat v = \Pi _\Phi ^XT(\hat v)$ w.r.t the oblique projection $\Pi _\Phi ^X$ can be represented as the oblique projection $\Pi _\Phi ^{{L^{\pi  \top }}X}$ of the true value function $V$, i.e., 
\begin{align}
\hat v = \Pi _\Phi ^XT^\pi(\hat v) = \Pi _\Phi ^{{L^{\pi \top}}X}V,
\end{align}
where $L^\pi = (I- \gamma P^\pi)$.
\label{lem:1}
\end{lemma}
\begin{proof}
Please refer to \cite{Scherrer:ObliqueProjection} for a detailed proof.
\end{proof}
\textbf{Remark}: Lemma~\ref{lem:1} helps to identify the equivalence between oblique projection of the true value function $V$, i.e., $\Pi _{\Phi }^{L^{\pi \top} X} V$ and the solution of the oblique projected fixed-point equation, i.e., $\hat v = \Pi _\Phi ^X T^\pi \hat v$. Figure~\ref{fig:oblique} is an illustration of the oblique projection.
\begin{figure}[h]
\centering
\includegraphics[width=.5\textwidth,height=1.75in]{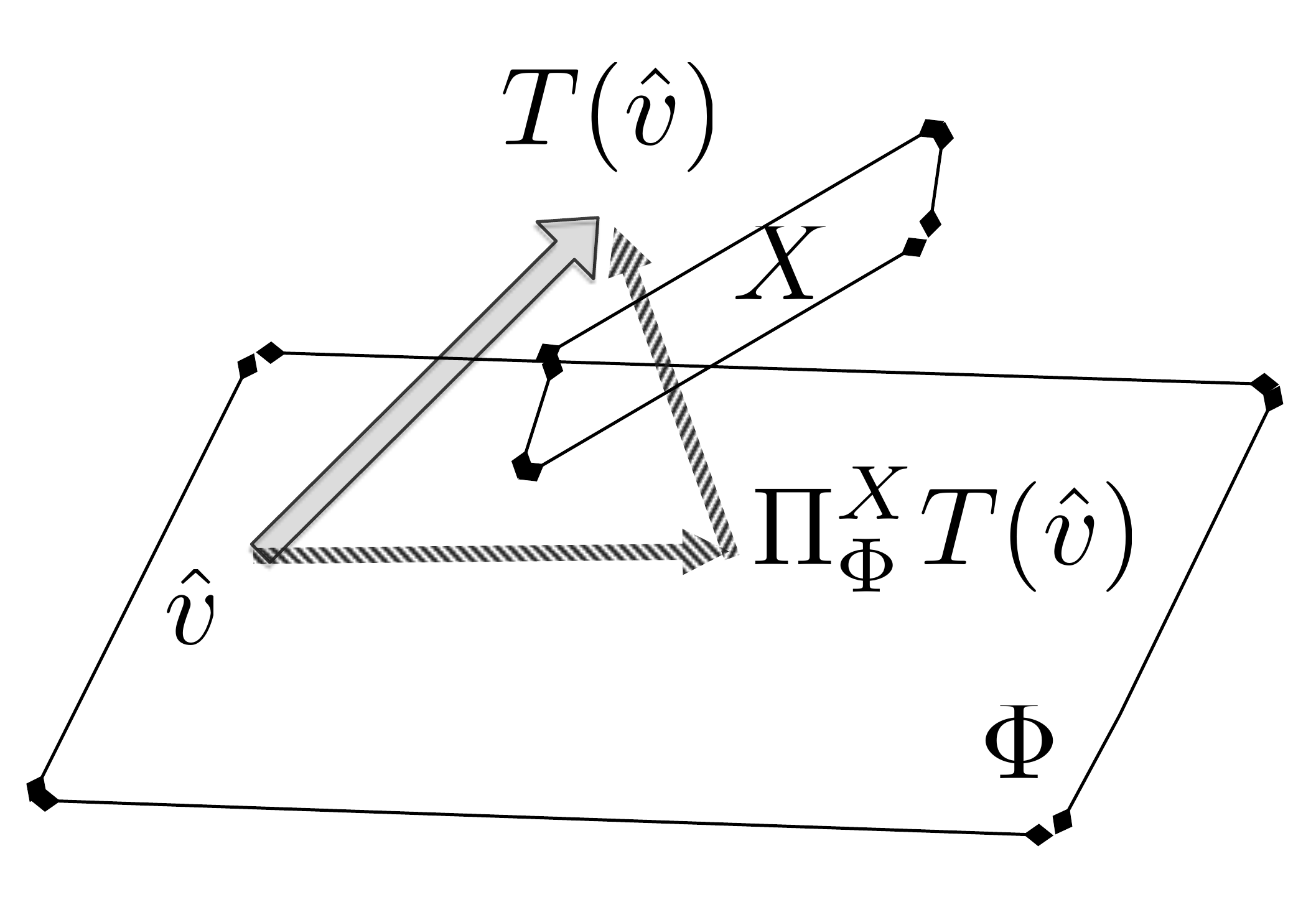}
\caption{An Illustration of Oblique Projected TD}
\label{fig:oblique}
\end{figure}

An intuitive question to ask is what the best oblique projection $X$ is. Is it either TD, RG, or some interpolation between them, or none of the above? To answer this question, we present the following proposition, which is the workhorse of this paper. 
\begin{lemma} \citep{Scherrer:ObliqueProjection}
\label{lem:psibest}
Given $\Phi$, if $V$ does not lie in $span(\Phi)$, the optimal approximation is $v^* = \Pi V = \Phi {({\Phi ^ \top }\Xi \Phi )^{ - 1}}{\Phi ^ \top }\Xi V$, and the corresponding oblique projection $X^*$ in the fixed point equation 
\begin{equation}
{v^*} = \Pi _\Phi ^{{X^*}}T{v^*}
\label{eq:v*}
\end{equation}
is 
\begin{align}
{X^*} = {({L^{\pi \top} })^{ - 1}}\Xi \Phi.
%= {({L^ {\pi -1}})^{ \top}}\Xi \Phi 
\label{eq:xoptimal}
\end{align}
\end{lemma}
\begin{proof}

From Lemma~\ref{lem:1}, we know that $X^*$ satisfies ${v^*} = \Pi _\Phi ^{({L^{\pi  \top }}){X^*}}V$. Let $\Pi _\Phi ^{({L^{\pi  \top }}){X^*}} = \Pi $, we have
\begin{equation}
({L^{\pi \top} }){X^*} = \Xi \Phi,
\end{equation}
and thus we can have Eq.~\eqref{eq:xoptimal}, which completes the proof.
\end{proof}
%

% \note[wd]{I think it is better to state this proposition with two clauses --- you are trying to accomplish too much in one sentence. 1. $\Pi^{L^\pi X^*}_\Phi v = \hat v^*$, and 2. $\hat v^*=\Pi_\Phi^{X^*}\mathcal T^\pi\hat v^*$, following the notation of Proposition 2 of \citep{Scherrer:ObliqueProjection}. In the proof, how you get $\hat v$ (right-hand side and clause 2 in my note) is not understandable if you do not state you first write down $\Pi_\Phi^X\mathcal T \hat v=\hat v$ and then solve for $\hat v$. By the way, your $\Pi^{X^*}_\Phi$ should be $\Pi^{L^\pi X^*}_\Phi$ according to your proof in the appendix and what should make the proposition hold.} 
%\note[lb]{Dong, is this presentation better now?}

%This proof is mainly based on the oblique projection introduced by \cite{Scherrer:ObliqueProjection}.
Although the analytical formulation of $X^*$ is clear, it is intractable to compute. The major reason is that  ${(L^{\pi \top})^{-1} }$ is computational prohibitive since the exact $P^\pi$ is not known. 
This paper will present techniques to compute $X^*$ approximately in the following.

%%%%%%%%%%%%%%%%%%%%%%%%%%%%%%%%%%%%%%%%%%%%%%%%%%
%%%%%%%%%%%%%%%%%%%%%%%%%%%%%%%%%%%%%%%%%%%%%%%%%%
%%%%%%%%%%%%%%%%%%%%%%%%%%%%%%%%%%%%%%%%%%%%%%%%%%
\section{Algorithm Design}
\label{sec:alg}
Given the knowledge of the oblique projection and the problem of the computational intractability to compute $X^*$, a criterion is proposed to approximate $X^*$. Based on this criteria, two algorithms are proposed. The first is based on state-aggregated two-stage approximation, and the second is based on state-dependent diagonalized approximation.

\subsection{Approximate Criteria}
Before presenting the algorithm design, we first introduce a simple but important property of the optimal projection matrix  $X^*$. Notice since $X^* = {(L^{\pi \top})^{-1} }\Xi \Phi $, and thus we have
\begin{equation}
{\Delta ^ \top }{X^*} = {\Phi ^ \top }({L^{\pi  \top }})({({L^{\pi  \top }})^{ - 1}}\Xi \Phi ) = {\Phi ^ \top }\Xi \Phi  = C.
\label{eq:p1}
\end{equation} 
Motivated by this, Proposition~\ref{pro:fundamental} is presented to formulate the cornerstone of this paper.
\begin{proposition}
For state aggregated $\hat{\Delta}$, there is
\begin{align}
\mathbb{E}_{\pi_b}[{\hat \Delta   }] = L^\pi \Phi,
\label{eq:fundamental}
\end{align}
and thus for the optimal oblique projection $X^*$ and the corresponding $v^*=\Phi \theta^*$, the following holds
\begin{align}
\mathbb{E}_{\pi_b}[{\hat \Delta   }]^\top {X^*}  
&= \mathbb{E}_{\xi}[\hat{C}]
\\
X^{*\ \top}\mathbb{E}_{\pi_b}[\hat\Delta]\theta^*
&= X^{*\ \top}\mathbb{E}_{\xi}[\hat R].
\label{eq:theta}
\end{align}
\label{pro:fundamental}
\end{proposition}
\vskip -3cm
\begin{proof}
Eq. (\ref{eq:fundamental}) is derived as follows,
\begin{align}
\mathbb{E}_{\pi_b}[{\hat \Delta }]
&= {\mathbb{E}_{{\pi _b}}}[{\rho _i}\Delta {\phi _i}]  \nonumber\\
\nonumber
&= {\sum\limits_{{a_i}} {{\pi _b}({a_i}|{s_i})\frac{{\pi ({a_i}|{s_i})}}{{{\pi _b}({a_i}|{s_i})}}} {{(\Delta {\phi _i})}^ \top }} \\
\nonumber
&= {\sum\limits_{{a_i}} {\pi ({a_i}|{s_i})} {{(\Delta {\phi _i})}^ \top }} \\
&= L^\pi \Phi.
\end{align}
Then we have
\begin{align}
\mathbb{E}_{\pi_b}[{\hat \Delta }^\top]X^*
\nonumber
&= {({L^\pi }\Phi )^ \top }{({L^{\pi  \top}})^ {-1} }\Xi \Phi  \\
&= {\Phi ^ \top }\Xi \Phi = C.
\end{align}
Insert Eq.~\eqref{eq:fundamental}, $\mathbb{E}_{\xi}[\hat{C}] = C$, and  $\mathbb{E}_{\xi}[\hat R]=R$ into Eq.~\eqref{eq:v*}, there is
\begin{align*}
\Phi\theta^* & =\Pi_{\Phi}^{X^{*}}(R+\gamma\Phi'\theta^*)\\
 & =\Phi(X^{* \top}\Phi)^{-1}X^{* \top}(R+\gamma\Phi'\theta^*)\\
\theta^* & =(X^{* \top}\Phi)^{-1}X^{* \top}(R+\gamma\Phi'\theta^*)\\
X^{* \top}\Phi\theta^* & =X^{* \top}(R+\gamma\Phi'\theta^*)\\
X^{* \top}(\Phi-\gamma\Phi')\theta^* & =X^{* \top}R\\
X^{* \top}\Delta\theta^* & =X^{* \top}R .
\end{align*}

This completes the proof.
\end{proof}

\subsection{Two-Stage State-Aggregated Batch Algorithm}
Motivated by Proposition~\ref{lem:psibest}, the following two-stage near-optimal off-policy TD algorithm is proposed, where the first step is 
\begin{align}
{\hat X} = \arg \mathop {\min }\limits_X \frac{1}{2}||{{\hat \Delta}^ \top }X - {\hat C}||_F^2.
\label{eq:hatx}
\end{align}
This problem is a well-defined convex problem, and there exists a unique solution.
When $(\hat{\Delta} \hat{\Delta}^\top)$ is nonsingular, the closed-form least-squares solution is computed as 
\begin{align}
\hat X  = {({\hat \Delta} {{\hat \Delta} ^\top})^ {-1} }({\hat \Delta} C).
\end{align}
On the other hand, if $({\hat \Delta} {\hat \Delta}^\top)$ is singular, which is more general, Eq.~\eqref{eq:hatx} can be solved via gradient descent method and can be further accelerated by Nesterov's accelerated gradient method \citep{nesterov2004introductory}.
The second step is to compute $\hat{\theta}$, i.e., 
\begin{align}
\hat \theta  = \arg \mathop { \min }\limits_\theta  ||{{\hat X}^ \top }(\hat \Delta\theta  - \hat{R})||_\xi^2.
\label{eq:hattheta}
\end{align}
This is a well-defined convex problem, and the solution is unique and can be easily solved via gradient descent method. When $({\hat X^ \top }\hat \Delta )$ is nonsingular, $\hat{X}$ can be simply solved via the one-shot least-squares solution 
\begin{equation}
\hat \theta  = {({{\hat X}^ \top }\hat \Delta)^{-1}}{{\hat X}^ \top }R.
\end{equation}
Based on these, we propose the \textit{State-aggregated Optimal TD Algorithm} (\textbf{SOTD}) as follows.

\begin{algorithm}
\label{alg:batch}
\caption{State-aggregated Optimal TD Algorithm (SOTD)}
\begin{algorithmic}[1]
\STATE INPUT: Sample set $\{ {\phi _i},{r_i},{\phi _i}^\prime \} _{i = 1}^n$
\STATE Compute $\hat{\Delta}, \hat{C}$, $\hat{R}$.
\STATE Compute $\hat X$ as in Eq.~\eqref{eq:hatx}.
\STATE Compute $\hat \theta$ as in Eq.~\eqref{eq:hattheta}.
\end{algorithmic}
\end{algorithm}

\subsection{State-Dependent Optimal Off-Policy TD learning}
Algorithm 1 can find the near-optimal projection matrix,
however, there is an apparent drawback of computing $\hat{X}$ in this way because of computational complexity.
Note that $\hat{X}$ is a $|\mathcal{S}| \times d$ matrix, which is computationally costly in large-scale reinforcement learning problems where the number of states $|\mathcal{S}|$ is large, or in continuous state space. 
% On the other hand, the ETD algorithm is computationally efficient with linear cost per step and does not require pre-knowledge of $n$, but require that the samples are collected sequentially.
To this end, the following algorithm is designed to tackle difficulties mentioned above.

%% motivation of state-weighted
In real applications where $d \ll |\mathcal{S}|$ or the state space is continuous, the proposed algorithm would not work well in practice since it has to compute a $|\mathcal{S}| \times d$ matrix $\hat{X}$. A desirable way out is to approximate the $(L^{\pi \top})^{-1}$ with a diagonal matrix $\Omega$, such that each row of $\Omega$ does not depend on other states, but only on its corresponding state. With such an assumption, $\hat{X}$ can be represented via a product of matrices as follows,
\begin{align}
\hat X = \Omega \Xi \Phi, 
\label{eq:hatxdiangonal}
\end{align}
where $\Omega$ is a $|\mathcal{S}| \times |\mathcal{S}|$ diagonal matrix. The $i$-th diagonal entry of $\Omega$ is denoted as $\omega_i$, i.e., $\Omega _{ii} := \omega_i$.
We term $\Omega$ as ``state-dependent'' diagonal matrix. 
Then the optimization problem reduces to
\begin{align}
\hat \Omega  = \arg \mathop {\min }\limits_\Omega  \frac{1}{2}||{\hat \Delta ^ \top }\Omega \Xi \Phi  - \hat C||_F^2,\quad{\rm{s.t.}}\quad{\Omega _{ij}} = 0,\quad i \ne j.
\label{eq:batch}
\end{align}

%By speculating the sample-separable structure of ${{\hat \Delta}^ \top }\Omega \Xi \Phi $ and ${\hat C}$ matrix, i.e, ${{\hat \Delta}^ \top }\Omega \Xi \Phi  = \frac{1}{n}\sum\limits_{i = 1}^n {{\rho _i}{\omega _i}{\phi _i}{{({\phi _i} - \gamma {{\phi '_i}})}^ \top }} $, ${\hat C}= \frac{1}{n} \sum\limits_{i = 1}^n {{\phi _i}{\phi_i ^\top}} $. 
%\begin{align}
%\Omega  = \arg \mathop {\min }\limits_\omega  ||\frac{1}{n}\sum\limits_{i = 1}^n {{\omega _i}{\rho _i}{\phi _i}} \Delta \phi _i^ \top  - \frac{1}{n}\sum\limits_{i = 1}^n {{\phi _i}} \phi _i^ \top ||_F^2
%\label{eq:sto1}
%\end{align}

%%%%%%%%%%%%%%%%%%
It is easy to prove the following
\begin{align}
{\hat \Delta ^ \top }\Omega \Xi \Phi {\rm{ }} = \mathbb{E}[{\rho _i}{\omega _i}{\phi _i}\Delta {\phi _i}^ \top ].
\end{align}
\note[lb]{Dong,I am per se, doing the loosest relaxation on norm minimization, aka, $||a+b|| < ||a|| +||b||$, but I want to give it some NICE justification. Do you have any ideas? Right now the justification is really too weak.}
\noindent
Based on the assumption that $\omega_{i}$ should be only (current) state-dependent, we have the following relaxed objective function, i.e., for the $i$-th sample, 
\begin{align}
\forall i,  {\omega _i} &= \arg \mathop {\min }\limits_\omega  || {\omega {\rho _i}{\phi _i}\Delta \phi _i^ \top  - {\phi _i}\phi _i^ \top }||_F^2.
\label{eq:sto2}
\end{align}
%%%%%%%%%%%%%%
% The solution can be computed as 
% \begin{align}
% {\omega _i} = \frac{{\sum\limits_{i,j} {\left( {({\rho _i}{\phi _i}\Delta \phi _i^ \top )^\circ ({\phi _i}\phi _i^ \top )} \right)} }}{{||{\rho _i}{\phi _i}\Delta \phi _i^ \top ||_F^2}},
% \label{eq:frosolver}
% \end{align}
% where $\circ$ is the Hadamard multiplication (entry-wise multiplication). The details of obtaining Eq.~\eqref{eq:frosolver} can be found in the Appendix. 
%%%%%%%%%%%%%%
Trace norm minimization can also be used, i.e., 
\begin{equation}
{\omega _i} = \arg \mathop {\min }\limits_\omega  ||{\phi _i}{(\omega {\rho _i}\Delta {\phi _i} - {\phi _i})^ \top }||_*.
\end{equation}

Two issues arise here:
\begin{itemize}
\item \textit{Computational cost}.
Trace norm minimization is usually more computationally expensive since it involves the singular value decomposition (SVD) operation.
\item \textit{Choice of the norm}. The issue here is to select the best norm as the objective function. Although there is already several pieces of literature discussing this problem, however, it remains unclear that at first glance, which norm minimization would achieve the best result in our problem.
\end{itemize}

% Actually, to utilize the symmetric structure of $\hat{C}$, an even better choice of to use trace-norm instead of the Frobenius norm as the measurement, such as follows,
% \begin{align}
% {\omega _i} = \arg \mathop {\min }\limits_\omega  || {\omega {\rho _i}{\phi _i}\Delta \phi _i^ \top  - {\phi _i}\phi _i^ \top } ||_*,
% \label{eq:sto3}
% \end{align}
%

We will resolve these two concerns by scrutinizing the structure of the problem. Notice that Eq.~\eqref{eq:sto2} can be written as ${\omega _i}= \arg \mathop {\min }\limits_\omega  ||{\phi _i}{(\omega {\rho _i}\Delta {\phi _i} - {\phi _i})^ \top }||_F^2$. 
Since ${\phi _i}{(\omega {\rho _i}\Delta {\phi _i} - {\phi _i})^ \top }$ is a rank-$1$ matrix, the solution is identical w.r.t Frobenius norm and trace norm, and the closed-form solution is
\begin{align}
{\omega _i} = \frac{{\Delta \phi _i^ \top {\phi _i}}}{{{\rho _i}\Delta \phi _i^ \top \Delta {\phi _i}}}.
\label{eq:sto4}
\end{align}
Interested readers will find a detailed deduction in the Appendix.
%%%%%%%%%%%%%%%%%%%%%%%%%%
%%%%% an ad-hoc approach
%%%We can make further approximation by utilizing the symmetric structure of  ${\hat C}$ matrix. A reasonable approximation way of Eq.~\eqref{eq:sto2} is using diagonal approximation
%%%\begin{align}
%%%{\omega _i} = \arg \mathop {\min }\limits_\omega  ||{\rm{diag}}\left( {\omega {\rho _i}{\phi _i}\Delta \phi _i^ \top  - {\phi _i}\phi _i^ \top } \right)||_F^2
%%%\label{eq:sto3}
%%%\end{align}
%%%%%%%%%%%%%%%%%%%%%%%%%%%
%  Given ${\omega _i}$ for the $i$-th sample, there are several ways to compute $\theta_i$. For example, the stochastic approximation way similar to ETD is as follows,
The update law is thus as follows,
\begin{align}
{\theta _{i + 1}} = {\theta _{i}} + \alpha_i \rho _i{\omega _i}{\delta _i}{\phi _i}.
\label{eq:theta-diag}
\end{align}
%
% or computed as the proximal gradient TD learning algorithm, which is detailed in the Appendix.
Samples with zero importance ratio (i.e., $\rho_i=0$) are discarded.
Now it is ready to formulate the \textit{Optimal Off-Policy TD Algorithm} (\textbf{O$^2$TD}) algorithm as in Algorithm 2. 
It is easy to verify that the computational cost per step is $O(d)$, as can be seen from the computation of Eq.~\eqref{eq:sto4} and \eqref{eq:theta-diag}.

\begin{algorithm}
\begin{algorithmic}[1]
\caption{Optimal Off-Policy TD Algorithm (O$^2$TD)}
\STATE INPUT: Sample set $\{ {\phi _i},{r_i},{\phi _i}^\prime \} _{i = 1}^n$
%\STATE collect a collection of $k=2000$samples \note[lb]{$k=2000$ is up to change.}, and compute $\hat C = \frac{1}{k} \sum\limits_{i = 1}^k {{\phi _i}{\phi_i ^\top}} $.
%\STATE Compute $\hat{\Delta}, \hat{C}$ as in Eq.~\eqref{eq:hatac}.
\FOR {$i=1,\ldots,n$}
\STATE Compute $\phi_i, \Delta \phi_i$, $\delta_i=r_i+\gamma\phi_i^{'\top}\theta_i-\phi_i^\top\theta_i$.
\STATE Compute $\omega_i$ according to Eq.~\eqref{eq:sto4}.
\STATE Compute $\theta_{i+1}$ according to Eq.~\eqref{eq:theta-diag}.
\ENDFOR
\label{alg:sto}
\end{algorithmic}
\end{algorithm}

%%%%%%%%%%%%%%%%%%%%%%%%%%%%%%%%%%%%%%%%
%%%%%%%%%%%%%%%%%%%%%%%%%%%%%%%%%%%%%%%%
%%%%%%%%%%%%%%%%%%%%%%%%%%%%%%%%%%%%%%%%
% \section{Theoretical Analysis}
% \label{sec:theory}

% In this section, we are going to discuss the 
% \note[lb]{think it over.}
% The problem formulation is as follows,
% given $||{A^\top}X - C||_F^2 \le \epsilon$, we need a perturbation analysis of $\sigma (CDE{D^ \top })$, where $D = {({X^ \top }L\Phi )^{ - 1}},E = {X^ \top }L\Xi {L^ \top }X$? As we showed previously, for the optimal $X^*$, we have
% \begin{itemize}
% \item ${D^*} = {C^{ - 1}}$
% \item $E^* = C$
% \end{itemize}
% and thus
% $CDE{D^ \top } = I$, and thus $\sigma(CD^*E^*{D^{* \top} }) = 1$. 

% Or the problem can be formulated as
% given $||{A^\top}X - C||_F^2 \le \epsilon$,
% what is the estimation of $||\Pi {L^ \top }X|{|_\xi }$?

\note[lb]{Dong,count on you!}

%%%%%%%%%%%%%%%%%%%%%%%%%%%%%%%%%%%%%%%%
%%%%%%%%%%%%%%%%%%%%%%%%%%%%%%%%%%%%%%%%
%%%%%%%%%%%%%%%%%%%%%%%%%%%%%%%%%%%%%%%%
\section{Related Work}
\label{sec:related}
%\subsection{State-weighted and Incremental Approximation}
% \subsection{Sequential State-dependent Diagonalized Approximation by Power Series Expansion}

One of the related work to optimal temporal difference learning is the emphatic temporal difference learning (ETD) work by \cite{etd:sutton2015}. That work was motivated by the off-policy convergence issue, and we will shed new light on the algorithms from the optimality perspective.
Similar to O$^2$TD, ETD also assumes that the optimal projection $X^*$  can be approximated by the product of diagonal matrices $\Omega, \Xi$ and the $\Phi$ matrices, i.e., the near-optimal projection matrix is formulated as in Eq.~\eqref{eq:hatxdiangonal}. Then a different technique is used based on the power series expansion, i.e., 
\begin{equation}
{({L^\pi })^{ - 1}} = {(I - {\gamma P^\pi })^{ - 1}} = \sum\limits_{i = 0}^\infty  {{{({\gamma P^\pi })}^i}}. 
\end{equation}
Then the power series expansion is used to compute $\Omega\Xi$ as a whole.  
Since the optimal oblique projection matrix is ${X^*} = {({L^{\pi \top} })^{ - 1}}\Xi \Phi $, it is evident that $\hat X = \Omega \Xi \Phi$ should be as close as possible to $X^*$, especially the diagonal elements. The diagonal elements of $\hat{X}$ are represented as a (column) vector $f$.
One conjecture is that for the diagonal matrix of $\hat{X}$, it is desired that
$
f = {({L^{\pi  \top }})^{ - 1}}{\xi }
$.
By using the power series expansion, $f$ can be expanded as
\begin{align}
f &= {({L^{\pi  \top }})^{ - 1}}{\xi } = (\sum\limits_{i = 0}^\infty  {{{(\gamma {P^{\pi \top}})}^i}} ){\xi }\\
 &= (I + \gamma {P^{\pi \top}} + {(\gamma {P^{\pi \top}})^2} +  \cdots + {(\gamma {P^{\pi \top}})^k} + \cdots){\xi }.
\end{align}
Readers familiar with the emphatic TD learning algorithm know that this is actually identical to Equation (13) in the paper by \cite{etd:sutton2015}, where a scalar follow-on trace is computed as
\footnote{We use subscript $\bullet_t$ to denote sequential samples, and subscription $\bullet_i$ to denote samples that are randomly sampled with replacement.}
\begin{align}
{F_0} = 1; \quad
{F_t} = I + \gamma {\rho _{t - 1}}{F_{t - 1}}, \quad t>0,
\label{eq:followon}
\end{align}
and it turns out that 
\begin{align}
{f_i} = {\xi }(i)\mathop {\lim }\limits_{t \to \infty } \mathbb{E}[{F_t}|{S_t} = {s_i}],
\end{align}
which will lead to the standard emphatic TD($0$) algorithm,
\begin{align}
\theta_{t+1} = \theta_t + \alpha_t F_t \rho_t \delta_t \phi_t.
\label{eq:theta_etd}
\end{align}

Due to space limitations, we refer interested readers to \citep{etd:sutton2015} for more details of the algorithm, and \citep{etd:hallak2015generalized, etd:yu2015weak} for more theoretical analysis.
It should also be noted that although this section does not provide any further extension of the ETD algorithm regarding algorithm design and analysis, to the best of our knowledge, it is the first time associating the ETD algorithm with near optimal temporal difference learning. This sheds a helpful light in understanding the family of the emphatic TD learning algorithms and the design of the follow-on trace. However, the ETD algorithm requires sequential sampling condition, i.e., $s'_t = s_{t+1}, \forall t>0$, which is not suitable for a set of samples collected from many episodes.

% For the completeness of the paper, we present the ETD algorithm here, which is a special case of the general ETD algorithm presented in \cite{etd:sutton2015}.
% %
% \begin{algorithm}
% \begin{algorithmic}[1]
% \label{alg:etd}
% \caption{A special case of ETD \citep{etd:sutton2015}}
% \STATE INPUT: Initial State $({\phi _0},{r_0},{\phi_1}) $
% %\STATE collect a collection of $k=2000$samples \note[lb]{$k=2000$ is up to change.}, and compute $\hat C = \frac{1}{k} \sum\limits_{i = 1}^k {{\phi _i}{\phi_i ^\top}} $.
% \FOR {$t=1,\ldots,n$}
% \STATE Compute $\phi_t, \Delta \phi_t$, $\delta_t=r_t+\gamma\phi_t^{'\top}\theta_t-\phi_t^\top\theta_t$.
% \STATE Compute follow-on trace $F_t$ according to Eq.~\eqref{eq:followon}.
% \STATE Compute $\theta_{t+1}$ according to Eq.~\eqref{eq:theta_etd}.
% \STATE Choose action $\pi_b(s_t)$, observe $r_t, s_{t+1}$.
% \ENDFOR
% \end{algorithmic}
% \end{algorithm}

%%%%%%%%%%%%%%%%%%%%%%%%%%%%%%%%%%%%%%%%
%%%%%%%%%%%%%%%%%%%%%%%%%%%%%%%%%%%%%%%%
%%%%%%%%%%%%%%%%%%%%%%%%%%%%%%%%%%%%%%%%
\section{Experimental Study}
\label{sec:experimental}
This section evaluates the effectiveness of the proposed algorithms. The effectiveness of SOTD algorithm is illustrated via comparison to LSTD, which is also a batch TD algorithm. A comparison study of O$^2$TD is conducted with GTD2 and ETD as three off-policy convergent TD algorithms with linear computational cost per step. 

\subsection{Experimental Study of SOTD}
\label{sec:exp-sotd}
The effectiveness of the proposed SOTD algorithm is shown by comparing the performance on the $400$-state Random MDP domain \citep{dann2014tdsurvey} with LSTD \citep{bradtke:1996linear,boyan:icml99} algorithm, which is one of the most sample-efficient algorithms to the best of our knowledge. 
Two widely used measurements in TD learning, Mean-Squares Projected Bellman Error (\textbf{MSPBE}) \citep{tdc:2009,dann2014tdsurvey} and Mean-Squares Error (\textbf{MSE}) are used as the error measurements.

This domain is a randomly generated MDP with $400$ states and $10$ actions \citep{dann2014tdsurvey}. The transition probabilities are defined as $P(s'|s,a) \propto p_{ss'}^a + {10^{ - 5}}$, where 
$p_{ss'}^a \sim U[0,1]$. 
The behavior policy $\pi_b$, the target policy $\pi$ as well as the start distribution are sampled in a similar manner. 
Each state is represented by a $201$-dimensional feature vector, where $200$ of the features were sampled from a uniform distribution, and the last feature was a constant one, the discount factor is set to $\gamma = 0.95$.
The number of features $d=200$, and we compare the performance of LSTD and SOTD with different numbers of training samples $n$, as shown in Figure~\ref{fig:sotd}. As Figure~\ref{fig:sotd} shows, with relatively small sample size $n$, SOTD tends to be even more sample-efficient than the LSTD algorithm.
\begin{figure}[h]
\centering
\begin{minipage}{1\textwidth}
\includegraphics[width=.5\textwidth,height=1.75in]{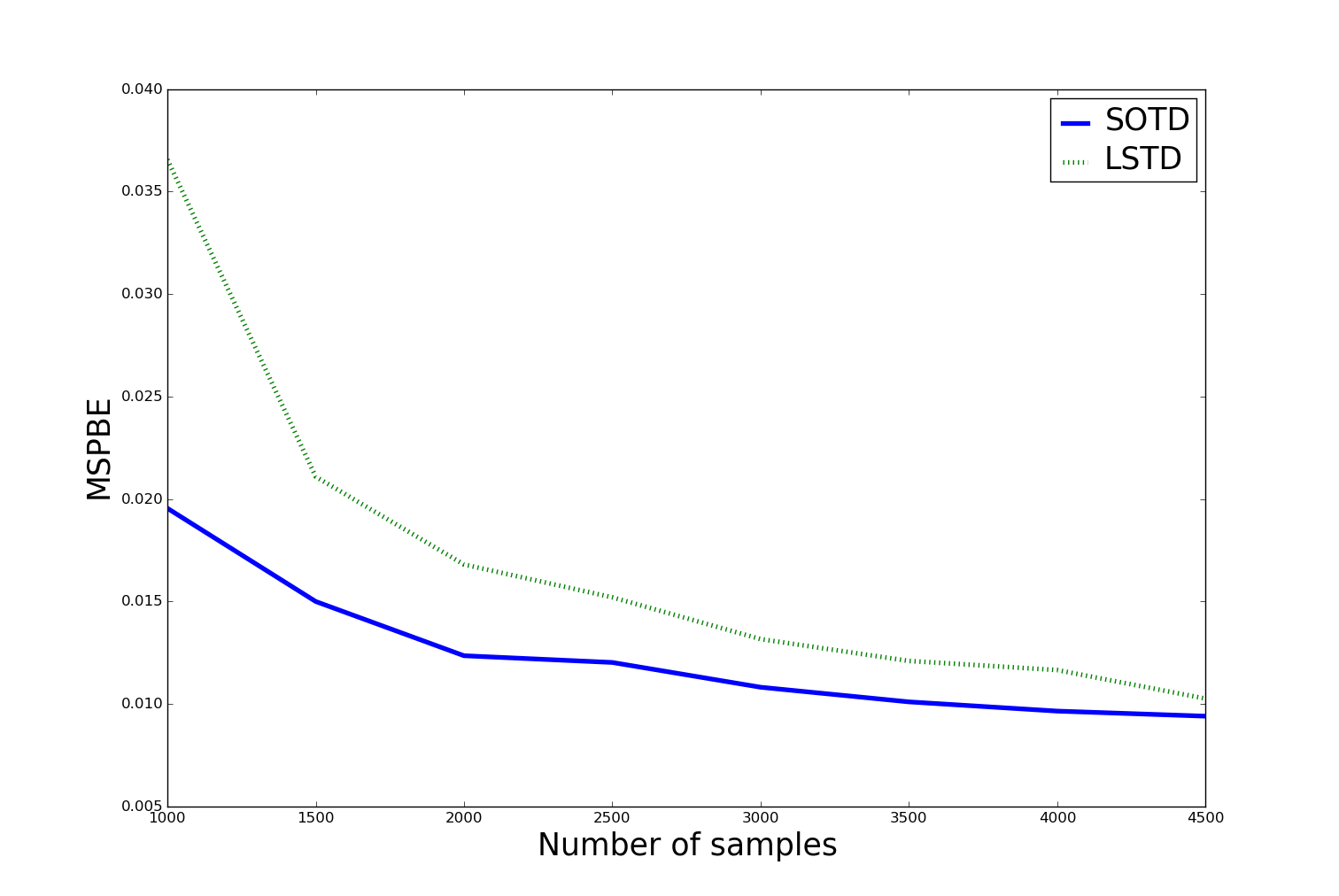}\\
\includegraphics[width=.5\textwidth,height=1.75in]{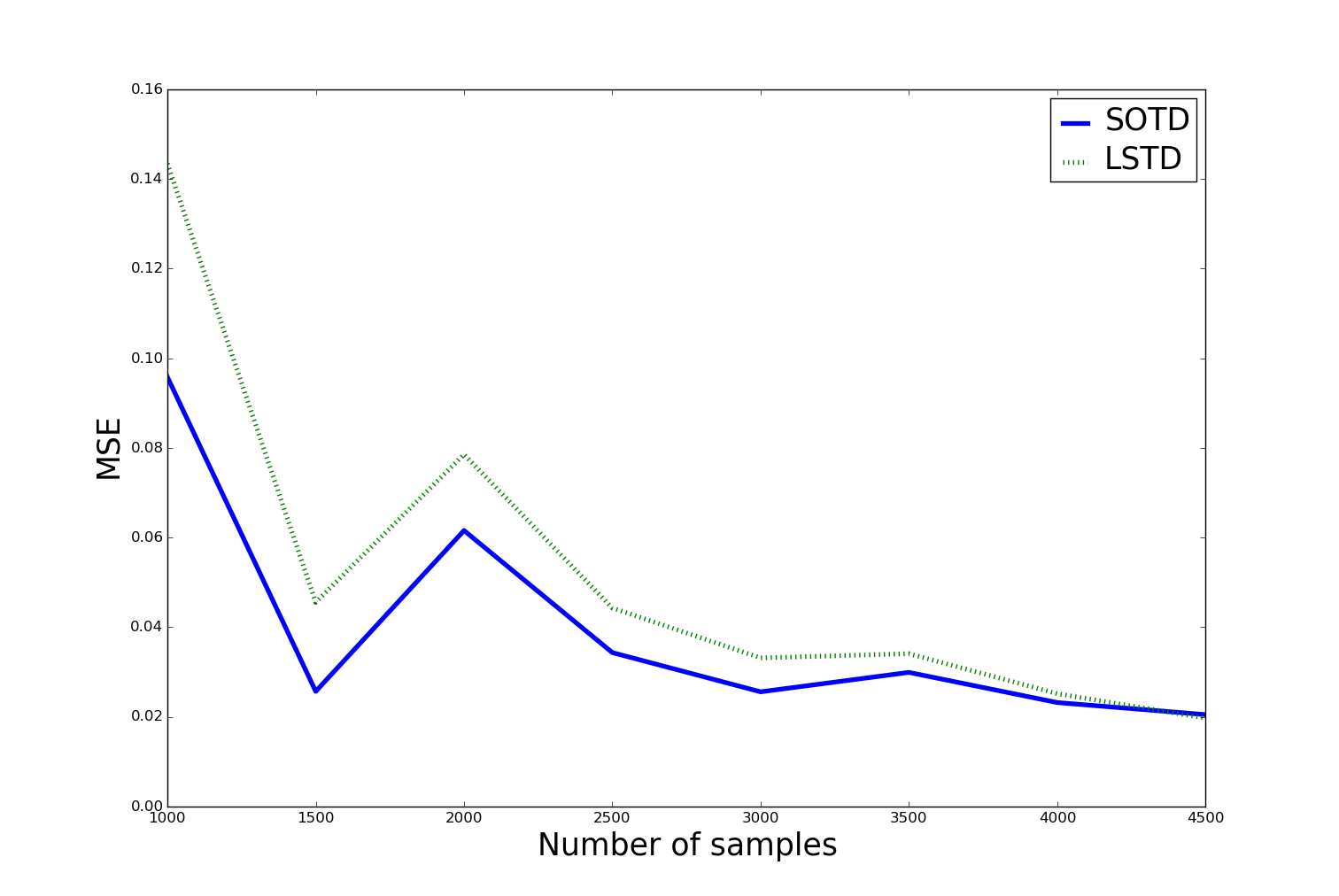}
\end{minipage}
\caption{Comparison between SOTD and LSTD on $400$-State Random MDP Domain}
\label{fig:sotd}
\end{figure}
%
% \begin{table}
% \centering
% \begin{tabular}{|c|c|c|c|}
% \hline 
% $n$ & Objective & LSTD & SOTD\tabularnewline
% \hline 
% $1000$& MSE   & $0.143$ & {$\bf{0.096}$}\tabularnewline
% \hline 
% $2000$& MSE   & $0.079$ & {$\bf{0.062}$}\tabularnewline
% \hline 
% $3000$& MSE   & $0.033$ & {$\bf{0.026}$}\tabularnewline
% \hline 
% \end{tabular}
% \caption{Comparison between SOTD and LSTD on $400$-State Random MDP Domain}
% \label{tab:1}
% \end{table}

%%%%%%%%%%%%%%%%%%%%%%%%%%%%%%%%%%%%%%%%
%%%%%%%%%%%%%%%%%%%%%%%%%%%%%%%%%%%%%%%%
%%%%%%%%%%%%%%%%%%%%%%%%%%%%%%%%%%%%%%%%
\subsection{Experimental Study of O$^2$TD}
\label{sec:exp-o2td}
This section compares the previous GTD2, ETD method with the O$^2$TD method using various domains with regard to their value function approximation performances. It should be mentioned that since the major focus of this paper is value function approximation and thus comparisons on control learning performance are not reported in this paper. We use $\alpha_{\rm{E}}$, $\alpha_{\rm{O}}$, and $\alpha_{\rm{G}}$ to denote the stepsizes for ETD, O$^2$TD, and GTD2, respectively.
Root Mean-Squares Projected Bellman Error (\textbf{RMSPBE}) and Root Mean-Squares Error (\textbf{RMSE}) are used for better visualization.

% \subsusection{Boyan Domain}

% We considered a chain of $14$ states $S = \{s_1, \cdots , s_{14}\}$ and one action. Each transition from state $s_i$ results in state $s_{i+1}$ or $s_{i+2}$ with equal probability and a reward of $-3$. If the agent is in the second last state $s_{13}$, it always proceeds to the last state with reward $-2$ and subsequently stays in this state forever with zero reward. $\gamma = 0.95$ and four-dimensional feature description with triangular-shaped basis functions covering the state space as described in \citep{dann2014tdsurvey}. The true value function, which is linearly decreasing from $s_1$ to $s_{14}$, can be represented by the features we choose, and the optimal parameter can be computed as $\theta^* = ???$.

% \begin{table}
% \centering
% \begin{tabular}{|c|c|c|}
% \hline 
%  & LSTD & SOTD\tabularnewline
% \hline 
% $||\theta  - {\theta ^*}|{|_\infty }$ & ???  & ???\tabularnewline
% \hline 
% $||\theta  - {\theta ^*}|{|_2 }$ & ???  & ???\tabularnewline
% \hline 
% MSPBE & ??? &  ???\tabularnewline
% \hline 
% MSE &  ???  &  ???\tabularnewline
% \hline 
% \end{tabular}
% \caption{Comparison between SOTD and LSTD on ??? Domain}
% \label{tab:1}
% \end{table}

\subsubsection{Baird Domain}
The Baird example \citep{Baird:ResidualAlgorithms1995} is a well-known example to test the performance of off-policy convergent algorithms. 
Constant stepsize $\alpha_{\rm{O}} = 0.006$, $\alpha_{\rm{G}} = 0.005$, which are chosen via comparison studies as in \citep{dann2014tdsurvey}.
The Monte-Carlo estimation of true value function $V$ is conducted as in \citep{dann2014tdsurvey}.
Figure~\ref{fig:baird} shows the RMSPBE curve and RMSE curve of GTD2, O$^2$TD of $5000$ steps averaged over $20$ runs. 
As can be seen from Figure~\ref{fig:baird}, although the variance of O$^2$TD is larger than GTD2's, O$^2$TD has a significant improvement over the GTD2 algorithm wherein the RMSPBE, the RMSE and the variance are all substantially reduced.
The low variance of the GTD2 learning curve can be explained by the advantage of stochastic gradient against stochastic approximation method, as explained in \cite{liu2015uai}.
\begin{figure}[htbp]
\centering
\begin{minipage}{1\textwidth}
\includegraphics[width=.5\textwidth,height=1.75in]{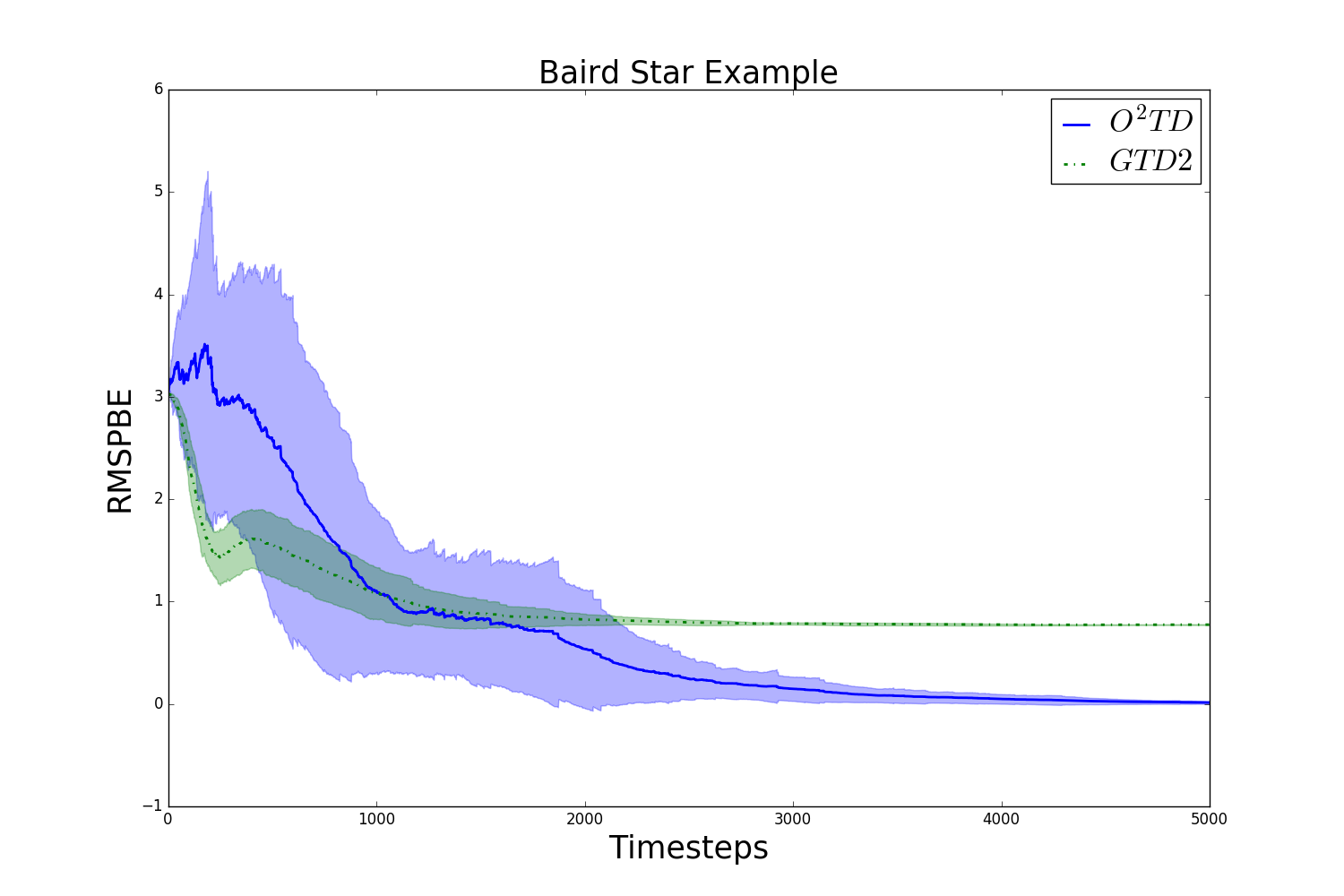}\\
\includegraphics[width=.5\textwidth,height=1.75in]{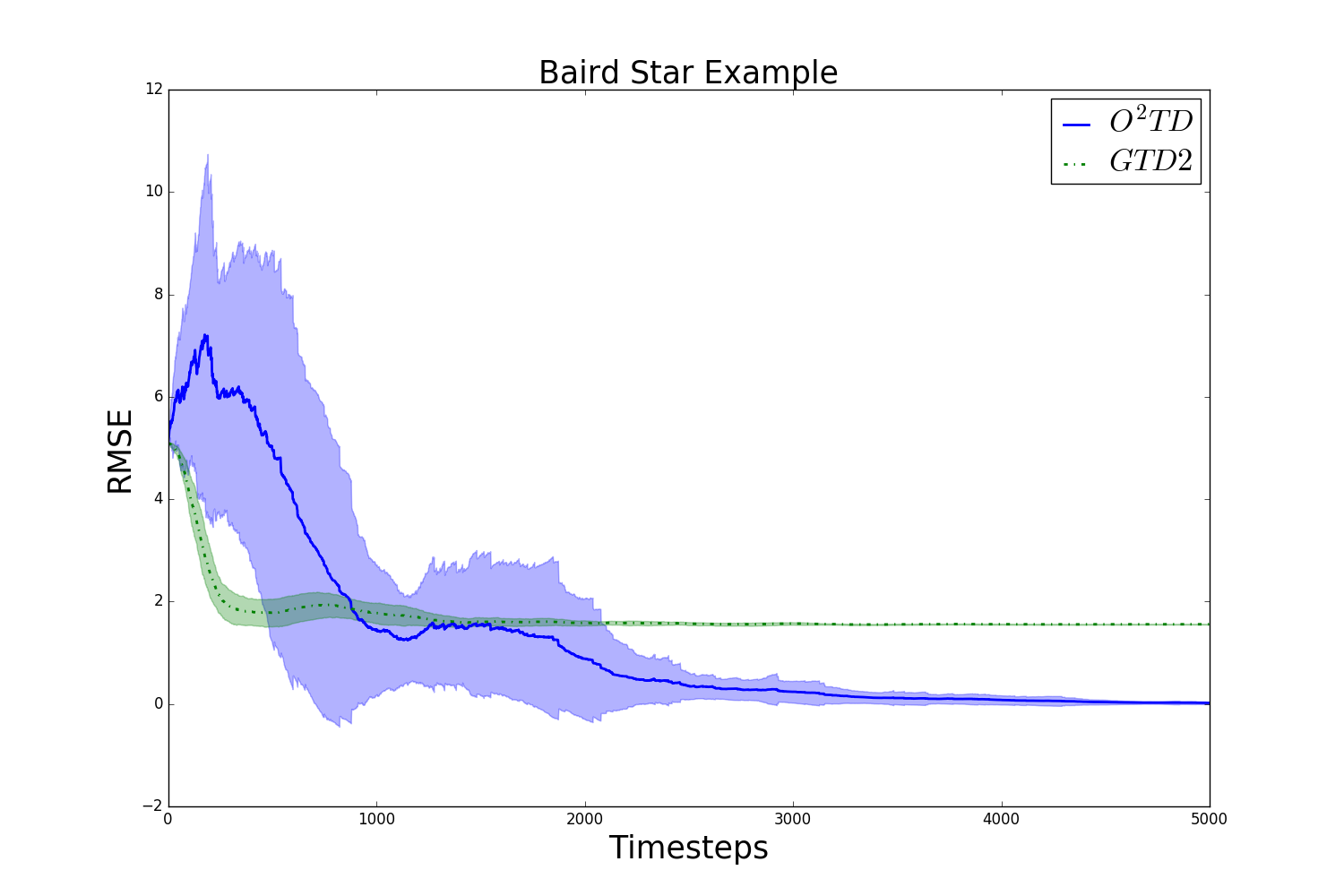}
\end{minipage}
\caption{Baird Domain}
\label{fig:baird}
\end{figure}

\note[lb]{Daoming, please add result description and analysis.}

\subsubsection{$400$-State Random MDP}
The randomly generated MDP with $400$ states and $10$ actions used in Section~\ref{sec:exp-sotd} is adopted as the second task. 
% The transition probabilities are defined as $P(s'|s,a) \propto p_{ss'}^a + {10^{ - 5}}$, where 
% $p_{ss'}^a \sim U[0,1]$. 
% The data-generating policy and start distribution were generated in a similar way. Each state
% is represented by a $201$-dimensional feature vector, where $200$ of the features were sampled from a uniform distribution, and the last feature was a constant one, the discount factor is set to $\gamma = 0.95$.
For sequential sampling (Figure~\ref{fig:ran-seq}), constant stepsize $\alpha_{\rm{E}} = 3*10^{-6}$, $\alpha_{\rm{O}} = 0.0007$, $\alpha_{\rm{G}} = 0.002$.
For random sampling (Figure~\ref{fig:ran-rd}), constant stepsize $\alpha_{\rm{E}} = 2*10^{-6}$, $\alpha_{\rm{O}} = 0.0006$, $\alpha_{\rm{G}} = 0.0009$. 
The Monte-Carlo estimation of true value function $V$ is conducted as in \citep{dann2014tdsurvey}. 
ETD tends to diverge easily with large stepsizes on this domain, so $\alpha_{\rm{E}}$ is set to be very small. As Figure~\ref{fig:ran-seq} and Figure~\ref{fig:ran-rd} show, O$^2$TD performs overall the best on this domain, although the variance is relatively larger than GTD2's.
\begin{figure}[htbp]
\centering
\begin{minipage}{1\textwidth}
\includegraphics[width=.5\textwidth,height=1.75in]{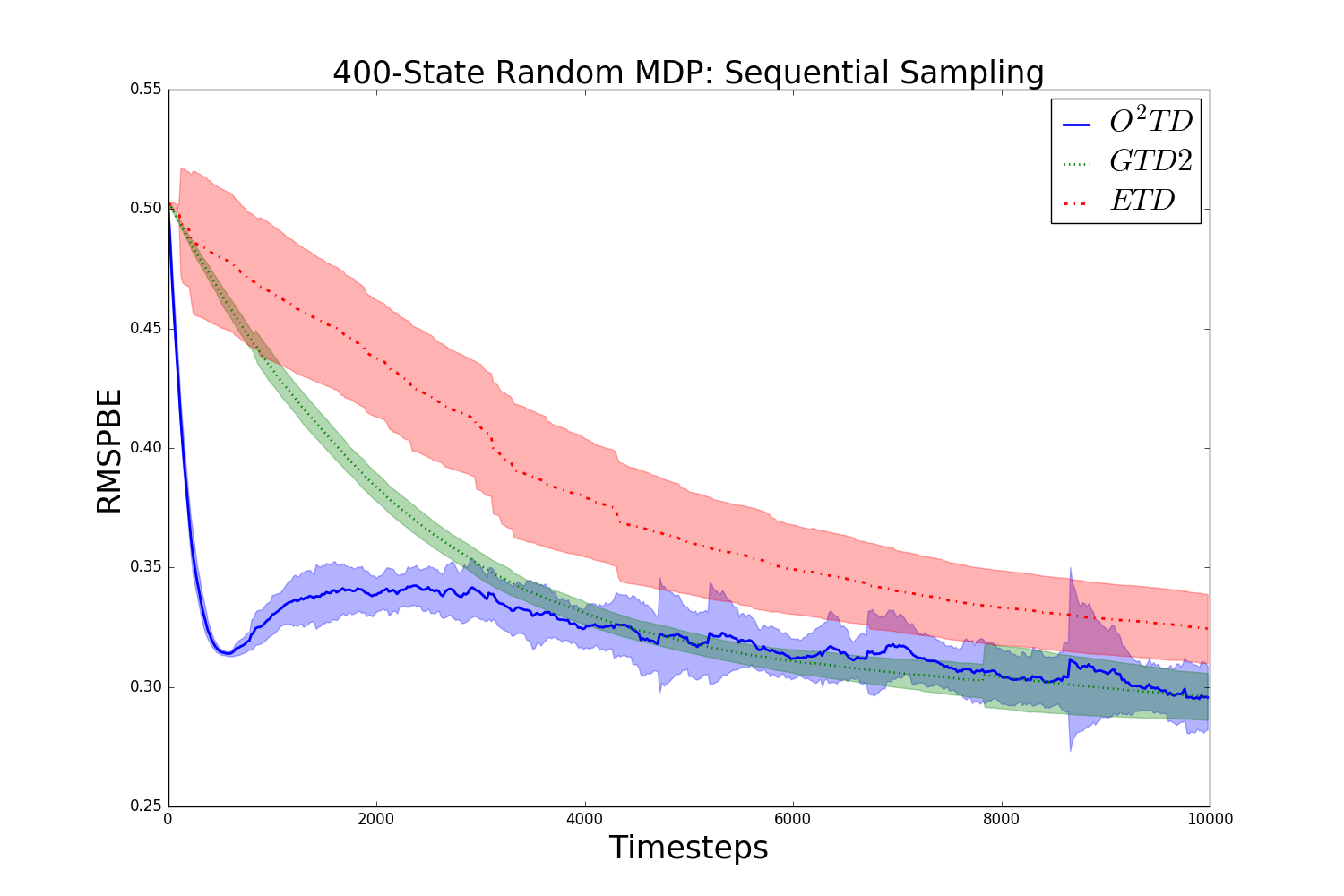}\\
\includegraphics[width=.5\textwidth,height=1.75in]{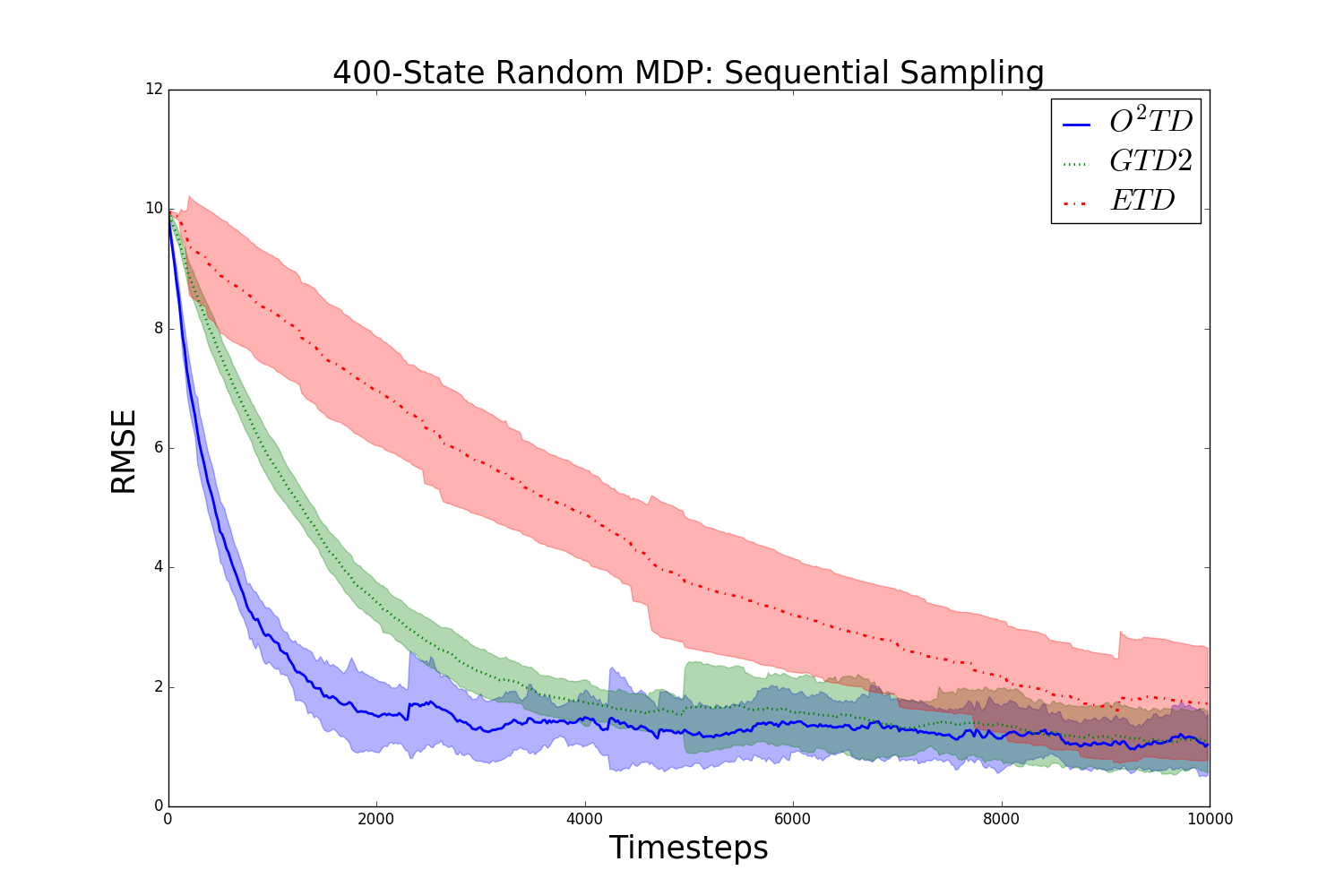}\\
\end{minipage}
\caption{Random MDP with Sequential Sampling}
\label{fig:ran-seq}
\end{figure}
\begin{figure}[htbp]
\centering
\begin{minipage}{1\textwidth}
\includegraphics[width=.5\textwidth,height=1.75in]{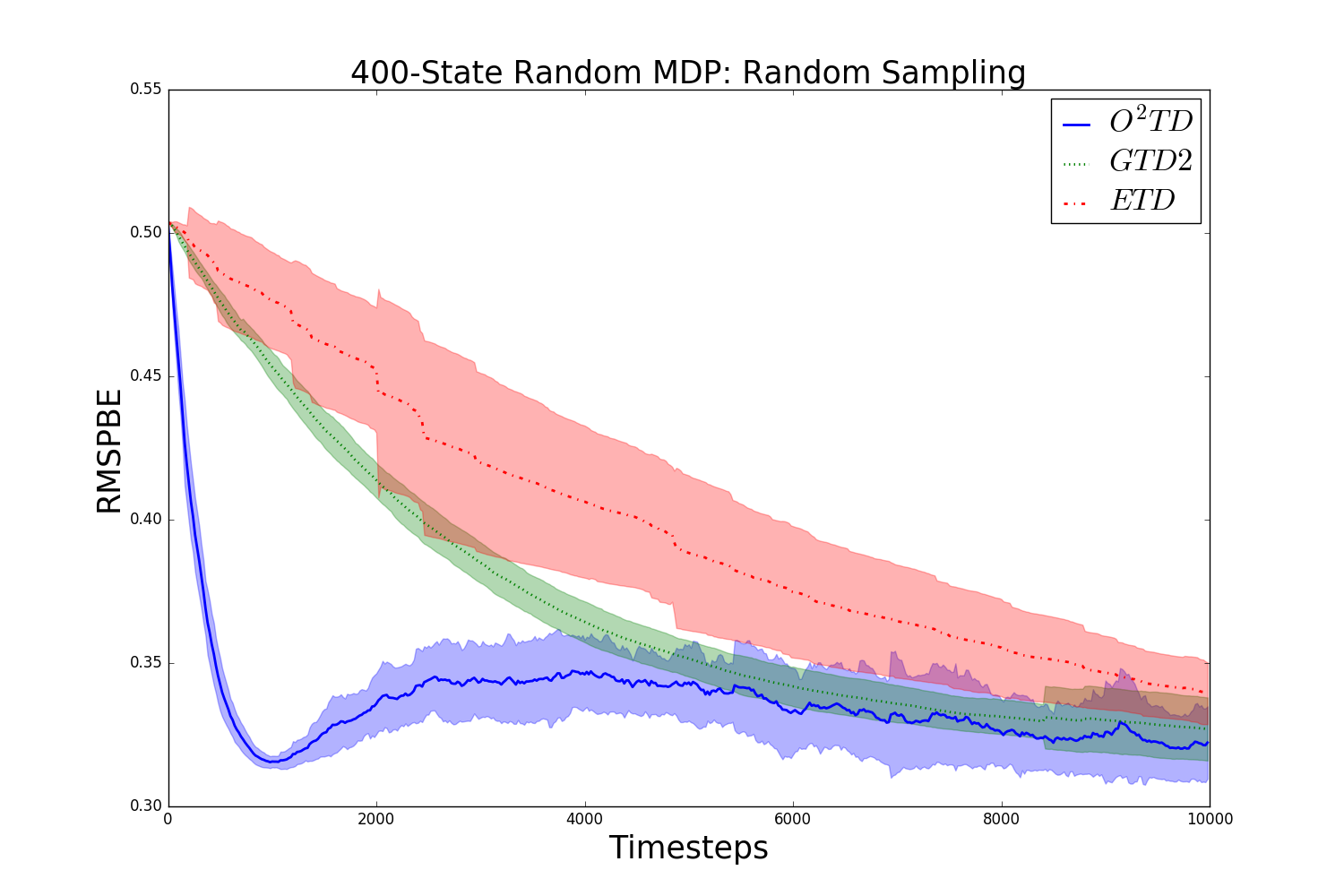}\\
\includegraphics[width=.5\textwidth,height=1.75in]{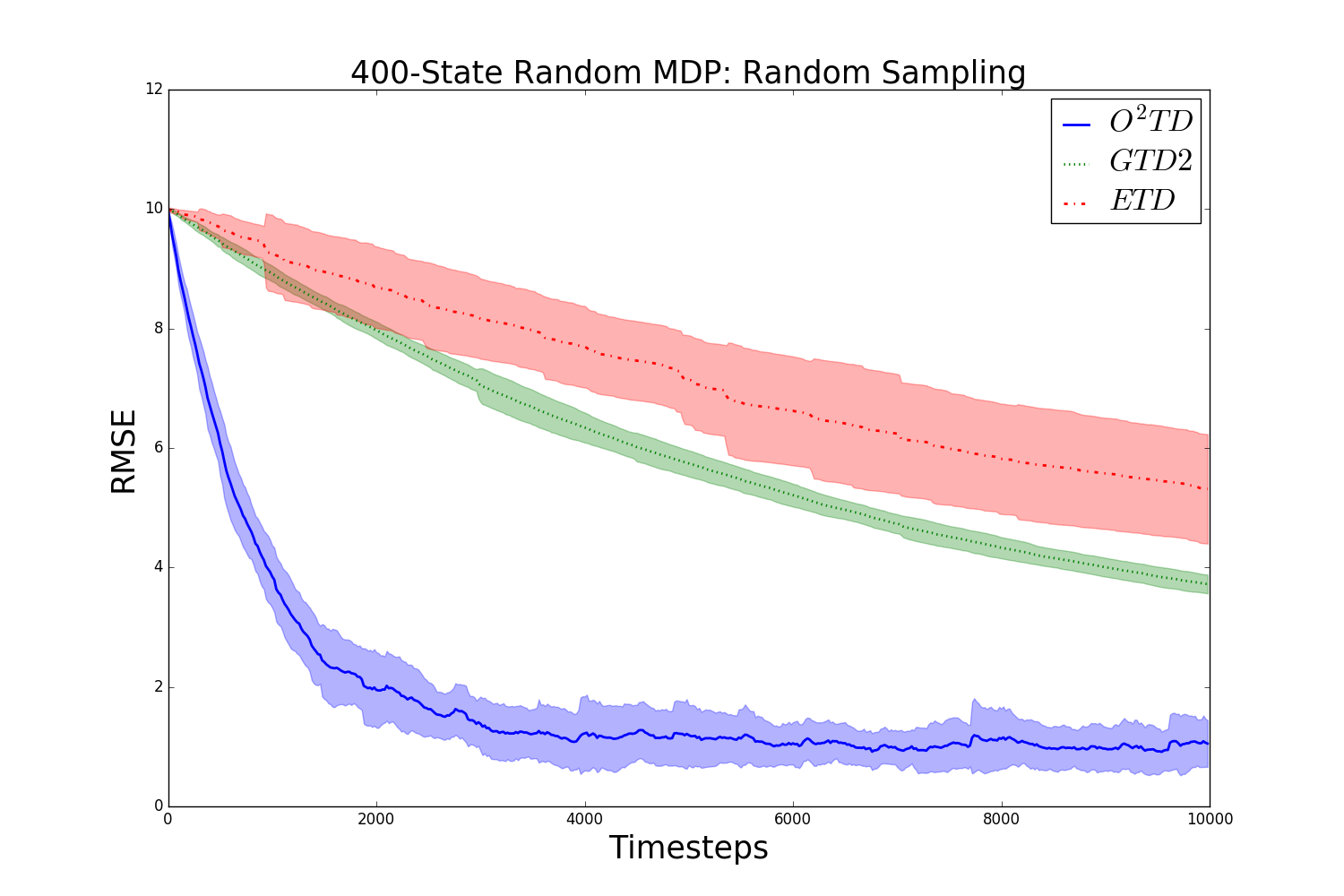}
\end{minipage}
\caption{Random MDP with Random Sampling}
\label{fig:ran-rd}
\end{figure}

\subsubsection{Mountain Car}
This section uses the mountain car example to evaluate the validity of the proposal algorithm. The mountain car MDP is an optimal control problem with a continuous two-dimensional state space. The steep discontinuity in the value function makes learning difficult. The Fourier basis \citep{konidaris:fourier} is used, which is a kind of fixed basis set.
An empirically good policy $\pi$ was obtained first, then we ran this policy $\pi$ to collect trajectories that comprise the dataset. On-policy policy evaluation of $\pi$ is then conducted using the collected samples.
For sequential sampling, constant stepsize $\alpha_{\rm{E}} = 0.001$, $\alpha_{\rm{O}} = 0.1$, $\alpha_{\rm{G}} = 0.2$.
For random sampling, constant stepsize $\alpha_{\rm{E}} = 0.0002$, $\alpha_{\rm{O}} = 0.05$, $\alpha_{\rm{G}} = 0.06$. 
The Monte-Carlo estimation of $V$ is estimated via $100$ runs.
As Figure~\ref{fig:mcar-seq} and Figure~\ref{fig:mcar-rd} show, GTD2 appears to perform the worst on this domain, and O$^2$TD tends to converge faster than ETD.
\begin{figure}[htbp]
\centering
\begin{minipage}{1\textwidth}
\includegraphics[width=.5\textwidth,height=1.75in]{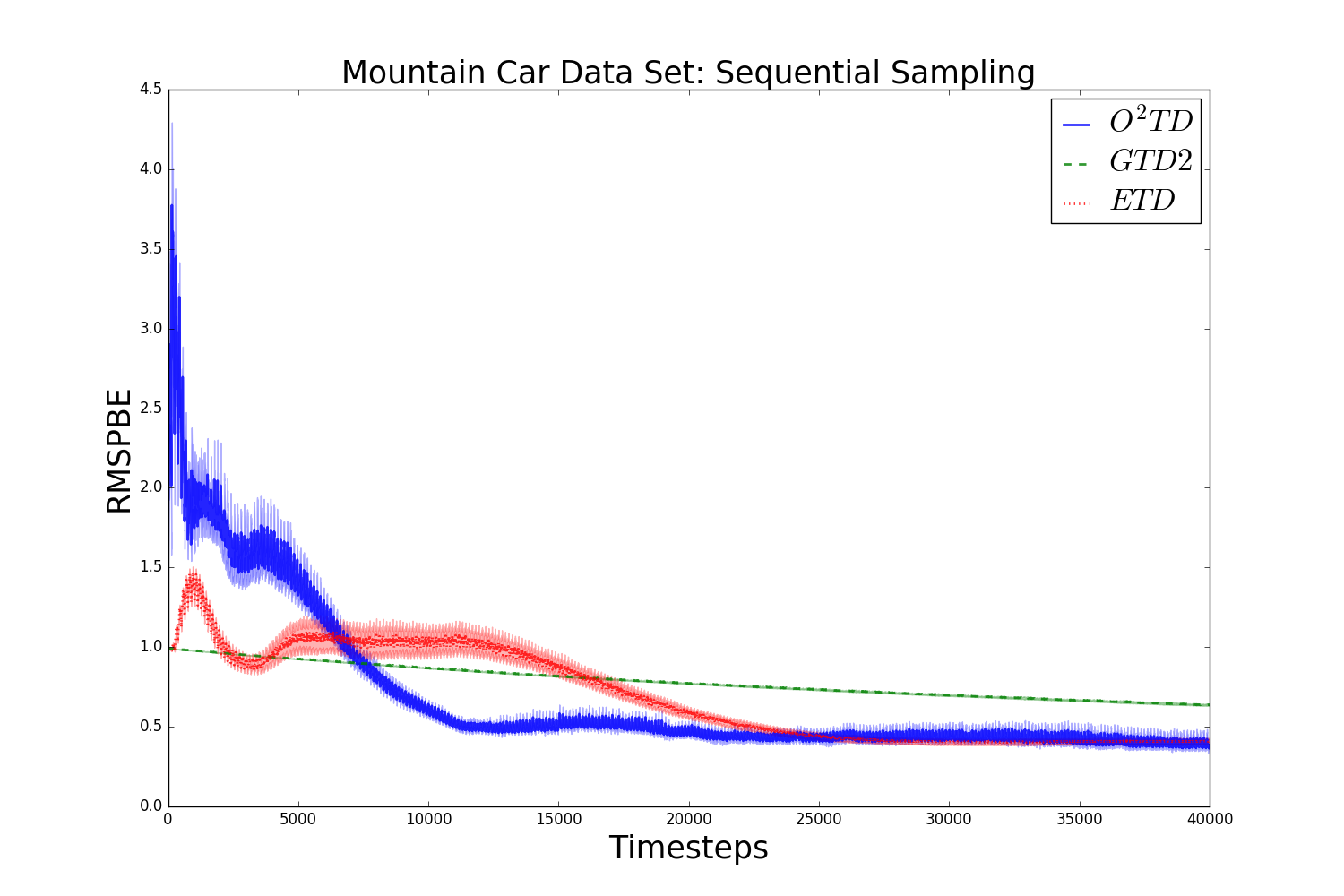}\\
\includegraphics[width=.5\textwidth,height=1.75in]{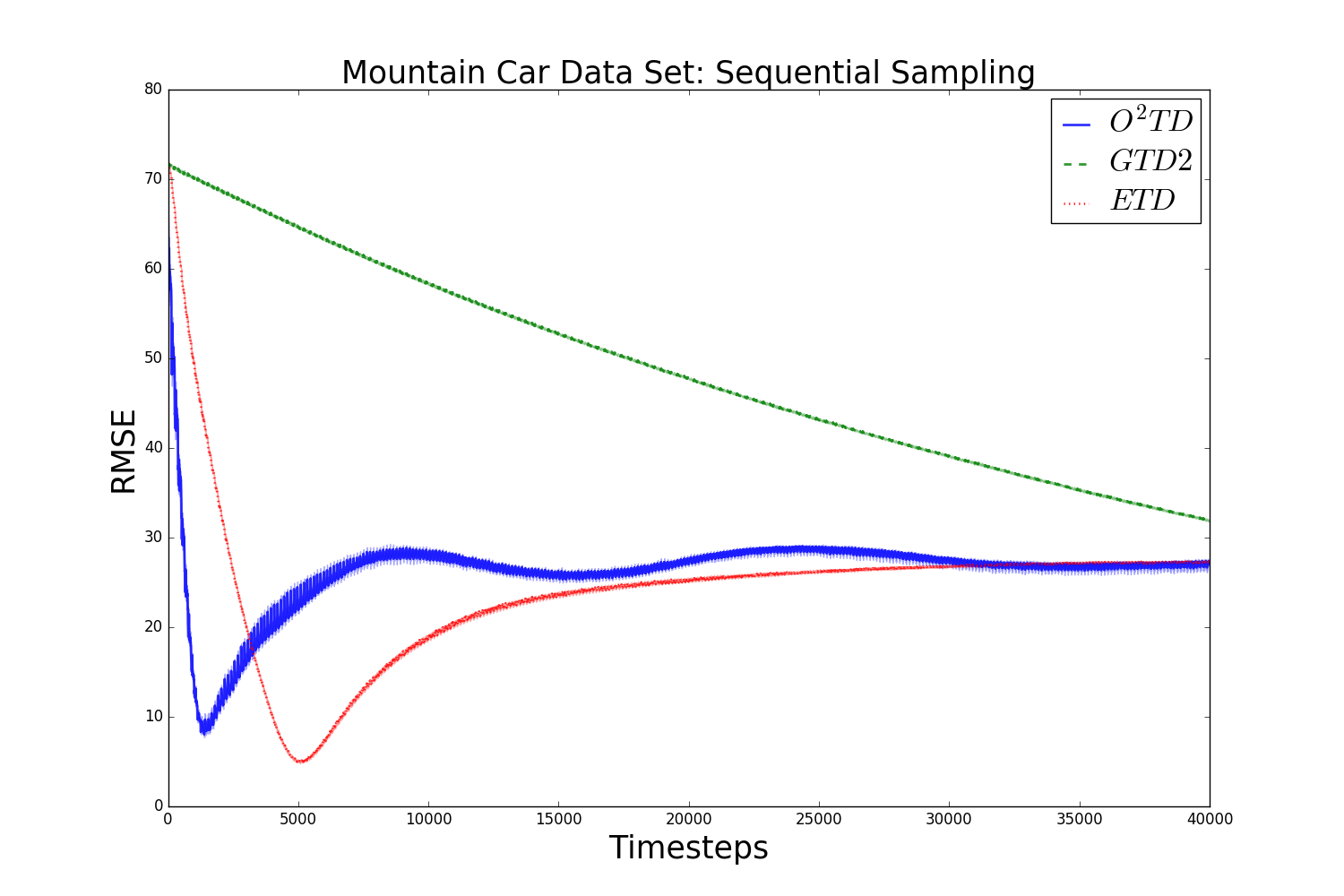}\\
\end{minipage}
\caption{Mountain car with Sequential Sampling}
\label{fig:mcar-seq}
\end{figure}
\begin{figure}[htbp]
\centering
\begin{minipage}{1\textwidth}
\includegraphics[width=.5\textwidth,height=1.75in]{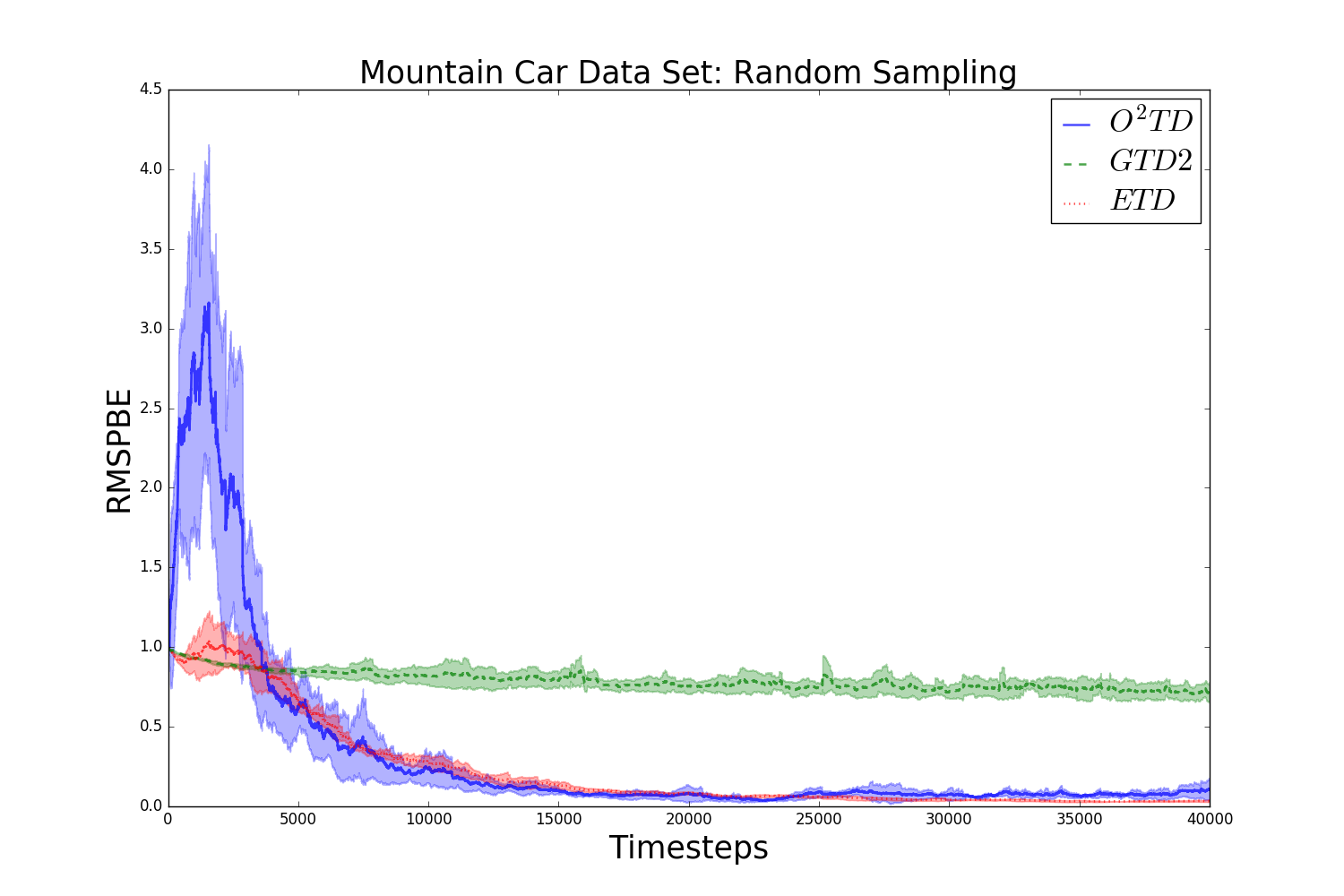}\\
\includegraphics[width=.5\textwidth,height=1.75in]{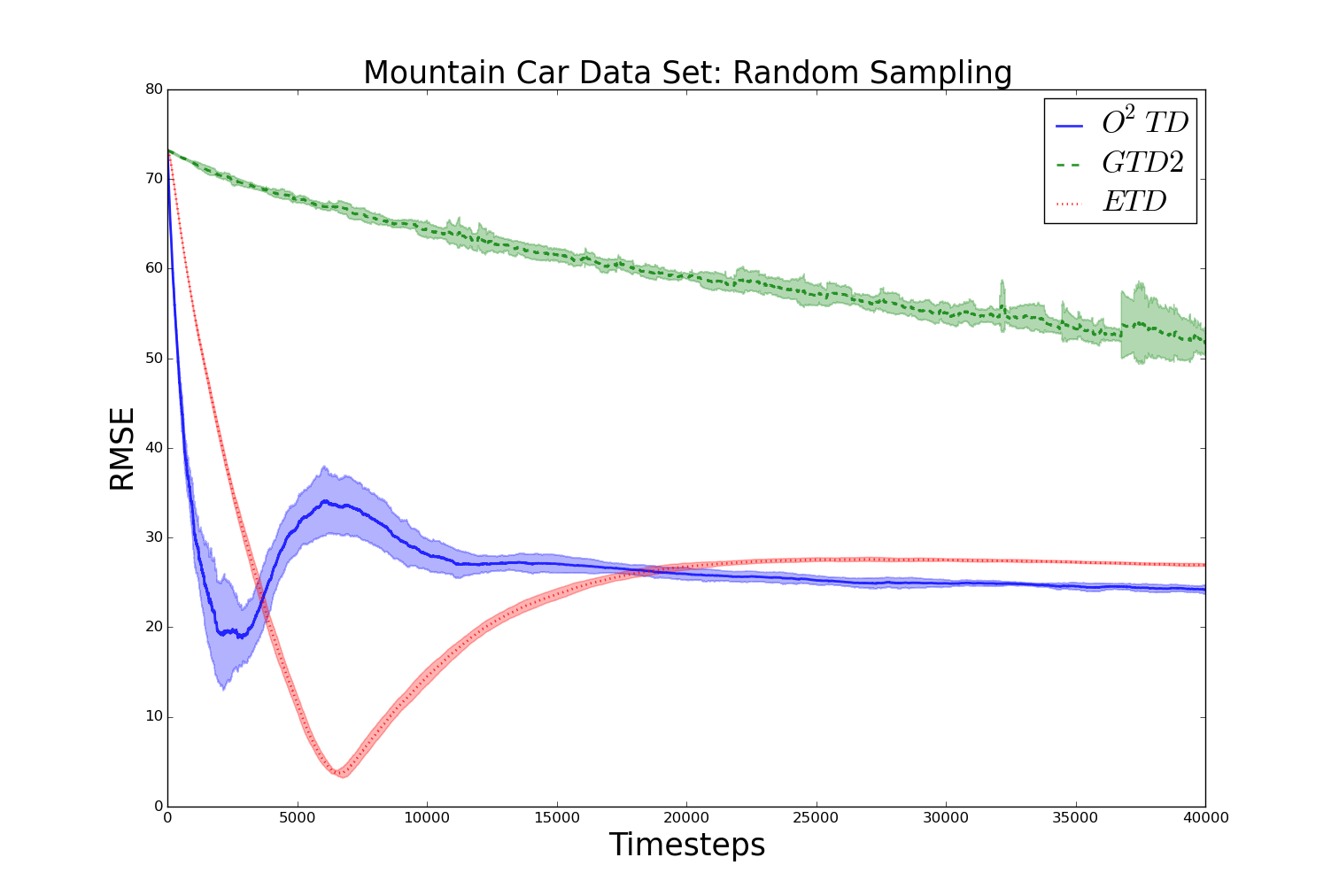}
\end{minipage}
\caption{Mountain car with Random Sampling}
\label{fig:mcar-rd}
\end{figure}

% \subsection{Energy Management Domain}
% % adapted from IJCAI paper, need to do tweaks in description
% In this experiment, we compare the performance of the proposed algorithms using an energy management domain.
% %We also report results on a variation of an energy arbitrage problem with storage, such as a battery. 
% The decision maker must decide how much energy to purchase or sell subject to stochastic prices. This problem is relevant in the context of utilities as well as in settings such as hybrid vehicles. The prices are generated from a Markov chain process. The amount of available storage is limited, and it also degrades with use. The degradation process is based on the physical properties of lithium-ion batteries and discourages fully charging or discharging the battery. The energy arbitrage problem is closely related to the broad class of inventory management problems, with the storage level corresponding to the inventory. However, there are no known results describing the structure of optimal threshold policies in energy storage.

% \note[lb]{Daoming, please add result description and analysis.}

% \begin{figure}[h]
% \centering
% \includegraphics[width=.5\textwidth,height=1.75in]{figures/baird_otd.png}
% \caption{Off-Policy Convergence Comparison}
% \label{fig:star}
% \end{figure}

%%%%%%%%%%%%%%%%%%%%%%%%%%%%%%%%%%%%%%%%%%%%%%%%%%%%%%%%%%%%%%%%%%%%%%%%%%%%%%
%%%%%%%%%%%%%%%%%%%%%%%%%%%%%%%%%%%%%%%%%%%%%%%%%%%%%%%%%%%%%%%%%%%%%%%%%%%%%%
%%%%%%%%%%%%%%%%%%%%%%%%%%%%%%%%%%%%%%%%%%%%%%%%%%%%%%%%%%%%%%%%%%%%%%%%%%%%%%
\section{Conclusion}

This paper proposes an interesting question: 
\begin{itemize}
\item How to improve the approximation quality of the true value function $V$?
\end{itemize}
To this end, several algorithms are proposed that can apply to different scenarios. Empirical experimental studies solidify the effectiveness of the proposed algorithm with different learning settings.

The major contribution is \textit{not} to propose another new TD algorithm with linear computational complexity per step, but to make an attempt to explore the optimal prediction of the value function in model-free policy evaluation. 
There are numerous promising future work potentials along this direction of research. One possible future research is to explore the relation between the near optimal projection matrix with eligibility traces and if the combination can improve the value function prediction performance in integration. Another interesting direction is that the current computationally tractable criteria of computing $X^*$ are based on Proposition~\ref{pro:fundamental} and the power series expansion of $(L^\pi)^{-1}$, it would be very intriguing to explore if there exist other computationally tractable criteria.

%%%%%%%%%%%%%%%%%%%%%%%%%%%%%%%%%%%%%%%%
%%%%%%%%%%%%%%%%%%%%%%%%%%%%%%%%%%%%%%%%
%%%%%%%%%%%%%%%%%%%%%%%%%%%%%%%%%%%%%%%%
%\newpage
\bibliographystyle{named}
\bibliography{thesisbib}

%%%%%%%%%%%%%%%%%%%%%%%%%%%%%%%%%%%%%%%%
%%%%%%%%%%%%%%%%%%%%%%%%%%%%%%%%%%%%%%%%
%%%%%%%%%%%%%%%%%%%%%%%%%%%%%%%%%%%%%%%%
 
\section*{Appendix}

\subsection*{Details of Eq.~\eqref{eq:sto4}}
To obtain Eq.~\eqref{eq:sto4}, we first introduce the following Lemmas to compute the singular value of rank-$1$ matrices.
\begin{lemma}
A rank-$1$ real square matrix $G=pq^\top$ where $p,q$ are vectors of the same length, the eigenvalues of $G$ are 
\begin{align}
\lambda (G) = \{ {p^\top}q,0,0,0, \cdots \},
\end{align}
i.e., $G$ has only one nonzero eigenvalue ${p^\top}q$, and all other eigenvalues are $0$, and thus we also have 
\begin{align}
{\rm{Tr}}(G) = {p^T}q,
\label{eq:tr}
\end{align}
where ${\rm{Tr}}(\cdot)$ is the trace of a matrix.
\label{lem:3}
\end{lemma}
Based on Lemma~\ref{lem:3}, we introduce Lemma~\ref{lem:2}.
\begin{lemma}
A rank-$1$ real matrix (not necessarily to be square) $M=uv^\top$ has only one nonzero singular value $\sigma_{\max} (M)  = ||u||_2 \cdot ||v||_2$, where $||\cdot||_2$ is the $\ell_2$-norm of a vector, and the Frobenius norm and the trace norm of $M$ are identical, i.e.,
\begin{align}
||M||{_*} = ||M||{_F} = \sigma_{\max} (M)  = ||u||_2 \cdot ||v||_2
\label{eq:equivalence}
\end{align}
\label{lem:2}
\end{lemma}

\begin{proof}
We use $M^H$ to represent the conjugate transpose of the $M$ matrix, and $\lambda (\cdot)$ to represent the eigenvalues of a square matrix, and $\lambda(\cdot)$ to represent the nonzero eigenvalue of a matrix. Then we have
\begin{align}
\lambda ({M^{\rm{H}}}M) 
\nonumber
&= \lambda (v{u^ \top }u{v^ \top }) \\
&= ({u^ \top }u)\lambda (v{v^ \top }) 
\end{align}
From Lemma~\ref{lem:3}, we know that $\lambda (vv^\top)$ are $\{ {v^ \top }v,0,0, \cdots \}$,
and thus $M$ has only one nonzero singular value ${\sigma _{\max }}(M)$, which is
\begin{align}
{\sigma_{\max } }(M) 
\nonumber
&= \sqrt {\lambda ({M^{\rm{H}}}M)} \\ 
\nonumber
&= \sqrt {{\lambda }(v{u^ \top }u{v^ \top })}  \\
\nonumber
&= \sqrt {({u^ \top }u){\lambda}(v{v^ \top })} \\
\nonumber
&= \sqrt {({u^ \top }u)({v^ \top }v)}   \\
\nonumber
&= ||u||_2 \cdot ||v||_2,
\end{align}
and all other singular values of $M$ are $0$. Thus $||M||_* = ||M||_F = ||u||_2 \cdot ||v||_2$, which completes the proof.
\end{proof}

Based on Lemma~\ref{lem:2}, we now show the derivation of Eq.~\eqref{eq:sto4}.
To tackle the following trace norm minimization formulation,
\begin{align}
{\omega _i} = \arg \mathop {\min }\limits_\omega  ||\omega {\rho _i}{\phi _i}\Delta \phi _i^ \top  - {\phi _i}\phi _i^ \top ||{_*},
\label{eq:tracemin}
\end{align}
we need to utilize the structure of the rank-$1$ matrices. We have
\begin{equation}
\omega {\rho _i}{\phi _i}\Delta \phi _i^ \top  - {\phi _i}\phi _i^ \top  = {\phi _i}{(\omega {\rho _i}\Delta {\phi _i} - {\phi _i})^ \top },
\end{equation}
we denote $q_i(\omega) := {(\omega {\rho _i}\Delta {\phi _i} - {\phi _i})^ \top }$, and thus we have
\begin{align}
\nonumber
\omega_i &=
\arg \mathop {\min }\limits_\omega  ||{\phi _i}{q_i^\top}(\omega )||{_*}\\
\nonumber 
&= \arg \mathop {\min }\limits_\omega  ||{\phi _i}||_2 \cdot ||q_i(\omega)||_2 \\
&= \arg \mathop {\min }\limits_\omega  ||q_i(\omega)||_2 
\end{align}
The second equality comes based on Eq.~\eqref{eq:equivalence}, and the third equality is based on the fact that $||{\phi _i}||_2$ does not depend on $\omega$.
This is equivalent to the following,
\begin{equation}
\omega_i = \arg \mathop {\min }\limits_\omega  || \omega {\rho _i}\Delta {\phi _i} - {\phi _i}||^2_2
\label{eq:tr}
\end{equation}
%%%%%%%%%%%%%%%%%%%%%%%%%%%%%%%%%%%%%
On the other hand, if we use $||\cdot||^2_F$ instead of trace norm minimization as in Eq.~\eqref{eq:tracemin}, we have
\begin{align}
{\omega _i} =\arg \mathop {\min }\limits_\omega  ||{\phi _i}{q_i}(\omega )||_F^2,
\label{eq:fromin}
\end{align}
And since 
\begin{align}
||{\phi _i}{q^ \top }_i(\omega )||_F^2 
\nonumber 
&= {\rm{Tr}}({q_i}(\omega ){\phi ^ \top _i }{\phi _i}{q^ \top _i}(\omega ))\\ 
\nonumber 
&= ({\phi ^ \top _i}{\phi _i}){\rm{Tr}}({q_i}(\omega ){q^ \top _i}(\omega ))\\
\nonumber 
&= ({\phi ^ \top _i}{\phi _i})({q^ \top _i}(\omega ){q_i}(\omega ))\\
&=  ||{\phi _i}||_2^2||{q_i}(\omega )||_2^2.
\end{align}
The first equality comes from that for a matrix $M$, there is
\begin{align}
||M||_F^2 = {\rm{Tr}}({M^H}M).
\end{align}
The third equality comes from Eq.~\eqref{eq:tr}. Then we can see that problem~\eqref{eq:fromin} is also equivalent to Eq.~\eqref{eq:tr}, as verified by Lemma~\ref{lem:2}. 

By taking the gradient of the right hand-side of Eq.~\eqref{eq:tr}, we will have Eq.~\eqref{eq:sto4} as the final result.

\end{document}